\documentclass[format=acmsmall]{acmart}
\pdfoutput=1
\usepackage{tikz}
\usepackage{pgfplots}
\usepackage{pgfplotstable}
\usetikzlibrary{3d}
\usetikzlibrary{positioning}
\usetikzlibrary{calc,matrix,backgrounds,fit}
\pgfdeclarelayer{myback}
\pgfsetlayers{myback,background,main}
\usepackage{subfig}
\usepackage{subfloat}
\usepackage{graphicx}
\usepackage{filecontents}
\usepackage[utf8]{inputenc}
\usepackage[english]{babel}

\usepackage[ruled,linesnumbered]{algorithm2e}

\usepackage{booktabs} 
\newcommand{\R}{\mathbb{R}}

\usepackage[ruled]{algorithm2e} 

\usepackage{amsmath}
\usepackage{amssymb}
\usepackage{stmaryrd}

\newcommand\independent{\protect\mathpalette{\protect\independenT}{\perp}}
\def\independenT#1#2{\mathrel{\rlap{$#1#2$}\mkern2mu{#1#2}}}

\newcommand{\norm}[1]{\left\lVert#1\right\rVert}

\usepackage{subfig}
\usepackage{subfloat}
\usepackage{algorithm2e}
\SetKw{KwBy}{by}

\usepackage{amsthm}
\usepackage[utf8]{inputenc}
\usepackage[english]{babel}
\usepackage{blindtext}
\usepackage{amssymb}
\usepackage{tikz}
\usepackage{pgfplots}
\usepackage{pgfplotstable}
\usetikzlibrary{calc}
\usetikzlibrary{3d}
\usetikzlibrary{positioning}
\usepackage{pgfplots}
\usetikzlibrary{calc,matrix,backgrounds,fit}
\pgfdeclarelayer{myback}
\pgfplotsset{compat=1.14}
\pgfsetlayers{myback,background,main}

\SetAlFnt{\small}
\SetAlCapFnt{\small}
\SetAlCapNameFnt{\small}
\SetAlCapHSkip{0pt}
\IncMargin{-\parindent}
\DeclareMathOperator{\Tr}

\setcopyright{acmcopyright}
\acmJournal{TKDD}
\acmYear{2020} \acmVolume{1} \acmNumber{1} \acmArticle{1} \acmMonth{1} \acmPrice{15.00}\acmDOI{10.1145/3385654}

\definecolor{bblue}{HTML}{4F81BD}
\definecolor{rred}{HTML}{C0504D}
\definecolor{ggreen}{HTML}{9BBB59}
\definecolor{ppurple}{HTML}{9F4C7C}

\definecolor{csacd}{rgb}{0.9, 0.17, 0.31} 
\definecolor{cgcd}{rgb}{0.27, 0.51, 0.71}
\definecolor{cbcdp}{rgb}{0.04, 0.85, 0.32} 
\definecolor{cfhals}{rgb}{0.62, 0.0, 1.0} 
\definecolor{cfmu}{rgb}{1.0, 0.83, 0.0} 
\definecolor{capg}{rgb}{0, 0, 0} 
\definecolor{ccdtf}{rgb}{0.5, 0.44, 0.51} 

\begin{document}
\definecolor{turquoise}{rgb}{0 0.41 0.41}
\definecolor{rouge}{rgb}{0.79 0.0 0.1} 
\definecolor{mauve}{rgb}{0.6 0.4 0.8} 
\definecolor{violet}{rgb}{0.58 0. 0.41} 
\definecolor{orange}{rgb}{0.8 0.4 0.2} 
\definecolor{bleu}{rgb}{0.39, 0.58, 0.93} 
\title{Efficient Nonnegative Tensor Factorization via Saturating Coordinate Descent}

\author{Thirunavukarasu Balasubramaniam}
\orcid{1234-5678-9012-3456}
\affiliation{%
  \institution{Queensland University of Technology}
  \streetaddress{2 George Street}
  \city{Brisbane}
  \state{QLD}
  \postcode{4000}
  \country{Australia}}
\email{thirunavukarasu.balas@qut.edu.au}
\author{Richi Nayak}
\orcid{1234-5678-9012-3456}
\affiliation{%
  \institution{Queensland University of Technology}
  \streetaddress{2 George Street}
  \city{Brisbane}
  \state{QLD}
  \postcode{4000}
  \country{Australia}
}
\email{r.nayak@qut.edu.au}
\author{Chau Yuen}
\orcid{1234-5678-9012-3456}
\affiliation{%
 \institution{Singapore University of Technology and Design}
 \streetaddress{8 Somapah Road}
 \country{Singapore}}
\email{yuenchau@sutd.edu.sg}

\begin{abstract}
With the advancements in computing technology and web-based applications, data is increasingly generated in multi-dimensional form. This data is usually sparse due to the presence of a large number of users and fewer user interactions. To deal with this, the Nonnegative Tensor Factorization (NTF) based methods have been widely used. However existing factorization algorithms are not suitable to process in all three conditions of size, density, and rank of the tensor. Consequently, their applicability becomes limited. In this paper, we propose a novel fast and efficient NTF algorithm using the element selection approach. We calculate the element importance using Lipschitz continuity and propose a saturation point based element selection method that chooses a set of elements column-wise for updating to solve the optimization problem. Empirical analysis reveals that the proposed algorithm is scalable in terms of tensor size, density, and rank in comparison to the relevant state-of-the-art algorithms. 
\end{abstract}

%
%


%
%


\maketitle

\renewcommand{\shortauthors}{Balasubramaniam T. et al.}

\section{Introduction}

With the advent of tagging, sensor and Internet of Things (IoT) technologies, online interaction data can be easily generated with tagging or spatio-temporal information. Tensor models become the natural choice to represent a multi-dimensional dataset, for example as ($user$ $\times$ $image$ $\times$ $tag$) or ($user$ $\times$ $location$ $\times$ $time$) \cite{ifada2014tensor,kolda2009tensor,symeonidis2016matrix,balasubramaniam2019sparsity}. A Tensor Factorization (TF) based method decomposes the tensor model into multiple factor matrices where each matrix learns the latent features inherent in the usage dataset~\cite{kolda2009tensor,zheng2014diagnosing}. These factor matrices can be used to reconstruct (or approximate) the tensor to predict missing entries. These entries can be inferred as new items that have been generated based on correlations and dependencies between the data, and help in making recommendation generation and link prediction \cite{rendle2010pairwise,dunlavy2011temporal,ermics2015link,adomavicius2005toward}. A learned factor matrix reveals the latent features that are able to represent hidden patterns in the dataset in that particular dimension, and can help to understand people’s mobility patterns and usage \cite{symeonidis2013geosocialrec}. 

A TF process is challenging for three main reasons. Firstly, a tensor model tends to grow big as the data size increases. It especially becomes a problem for web applications due to a large Internet population and fewer user interactions. Secondly, due to fewer user interactions with all the items, the data generated is very sparse. Learning the associations or analyzing this sparse dataset is hard due to lack of correlation patterns. Thirdly, the complexity of factorization process increases with an increase in the rank of the tensor. Here, the rank represents the number of hidden features in the dataset, in other words, it decides the size of factor matrices. Computing a large size sparse tensor model for exploring hidden latent relationships in the dataset is still a challenging problem. Table~\ref{tab:one} shows the limited scalability of existing factorization algorithms. Nonnegative Tensor Factorization (NTF) which imposes the nonnegative constraint to the factor matrices learned during factorization is the backbone of recommendation and pattern mining methods to understand interactions between features. The development of a fast and efficient NTF algorithm is needed for developing effective web recommendation and pattern mining methods.  

We propose the Saturating Coordinate Descent (SaCD) factorization algorithm that reduces the complexity inherent in the factor matrix update by carefully selecting the important elements to be updated in each iteration. We propose a Lipschitz continuity based element importance calculation to effectively select the elements for an update that will lead fast convergence. We also propose a fast-parallelized version of SaCD (named as FSaCD) to further speed up the factorization process that suffers from Intermediate Data Explosion (IDE). IDE is caused due to the materialization and storage of the intermediate data generated. The column-wise element selection and element update minimizes the complexity of IDE and parallelization becomes easier. 

Extensive experiments with recommendation and pattern mining applications show improved scalability (i.e. size, density and rank) performance of SaCD and FSaCD in comparison to all relevant baseline algorithms. The results also show that the efficiency is achieved at no cost of accuracy. In fact, SaCD achieves high accuracy when compared to other benchmarks.

\begin{table}%
\caption{Scalability in terms of size (mode length), density (percentage of observed entries in the tensor) and rank of the tensor. Algorithms are ranked as low, medium, high, and very high based on their capability in executing the process without running out of memory or out of time for synthetic datasets used in Sections 6.2 and 6.6. SaCD and FSaCD are the proposed algorithms.}
\label{tab:one}
\begin{minipage}{\columnwidth}
\begin{center}
\begin{tabular}{llll}
\toprule
Algorithm & Mode length & Density & Rank\\
\toprule
APG \cite{zhang2016fast}  & Low  & Low  & Low\\
FMU \cite{phan2012fast} & Medium  & Low  & Medium\\
FHALS \cite{phan2013fast} & Medium  & Low  & Medium\\
BCDP \cite{xu2013block} & Medium  & Medium  & High\\
CDTF \cite{shin2017fully} & High  & Medium  & Low\\
GCD \cite{balasubramaniam2018nonnegative} & High  & Medium  & Low\\
SaCD & High  & High  & High\\
FSaCD  & High  & Very High  & High\\
\bottomrule
\end{tabular}
\end{center}
\bigskip\centering

\end{minipage}
\end{table}%

The specific contribution of this paper is four-fold. 
\begin{enumerate}
	\item We propose a novel Lipschitz continuity based element importance calculation method to select certain elements for updating in the factorization process.
	\item We propose a novel Saturating Coordinate Descent (SaCD) factorization algorithm that updates factor matrices by selecting elements using the proposed Lipschitz continuity based element importance and leads fast convergence to a local optimum solution.
	\item We apply SaCD in recommender systems and pattern mining effectively with high accuracy and scalability.
	\item We design FSaCD by parallelizing SaCD on the single machine using multi-cores that further improves the performance on big and dense datasets.
\end{enumerate}

To our best of knowledge, SaCD is the first algorithm that can efficiently process the large and sparse tensor model with large rank factor matrices, with and without parallelization.

\section{Related Work}
\textbf{Tensor Factorization based recommender systems and pattern mining methods}: Tensor models have been widely applied in several web-based applications and have shown superiority in generating latent patterns \cite{ifada2016relevant,ermics2015link,sun2016understanding}. The majority of model-based Collaborative Filtering methods, including Matrix Factorization (MF), fail to utilize a context along with users and items information because of their limited capability to deal with two-dimensional data only \cite{symeonidis2016matrix}. The capability of a tensor model to represent multidimensional data makes them a natural choice to represent the user interaction data \cite{kutty2012people}. The last decade has witnessed the rise of tensor-based recommendation methods \cite{ifada2014tensor,symeonidis2016matrix,yu2015survey,zheng2010photo,symeonidis2008tag,rendle2010pairwise}. 

A three dimensional tensor model represents the data as ($user$ $\times$ $item$ $\times$ $context$) and a tensor factorization (TF/NTF) algorithm can be used to capture dimensional dependencies and general recommendations~\cite{karatzoglou2010multiverse}. For example, a ($user$ $\times$ $item$ $\times$ $tag$) tensor model has been used in social tagging systems for representing how users have used tags with items~\cite{bouadjenek2016social,ifada2016relevant}. A TF method with higher order singular value decomposition (HOSVD) has been used to recommend items based on tagging behavior or to recommend tags for items~\cite{symeonidis2009unified}. Link prediction in social networks has been solved using Alternating Least Square (ALS) based NTF by predicting the missing relations in the dataset, for example, identifying relations among users or recommending items to users by considering time as the additional context \cite{ermics2015link}.  Researchers have used the $p-core$ approach to minimize the size and sparsity of the tensor by only including the data with occurrence of (user, item) at least $p$ times, as the traditional TF/NTF methods are not scalable and their accuracy on the sparse dataset is low \cite{ifada2016relevant,rendle2010pairwise,balasubramaniam2019sparsity}. However, the $p-core$ approach removes any users/items with less than $p$ interactions, thus, leaving the recommendation process incomplete for some users.

Recent advancements in IoT allow recording additional spatio-temporal information easily in addition to the web interactions. This has led to focus research on Location-Based Social Networks (LBSNs) \cite{yu2015survey}. Researchers have successfully applied TF/NTF using ALS for automated mining of temporal patterns \cite{balasubramaniam2018nonnegative,liu2013point}. For example, the Foursquare dataset with the user, venue and time is represented as a tensor and the $r$ number of patterns are identified by the factor matrix in the time mode where $r$ is a tensor rank \cite{zheng2014diagnosing}. Similarly, the temporal trajectories of online gamers are derived using NTF \cite{sapienza2018non}. These patterns allow learning the behavior of web users that helps to personalize and improve the interactions. Tensor rank is an important parameter in the factorization process that learns the hidden features of the dataset. These hidden features represent the data effectively in a lower dimension, where higher the number of hidden feature identified, higher is the chances of learning the true representation of dataset \cite{symeonidis2016matrix}. 

With the dependence of these methods on TF to learn the associations in the data, it is important that a factorization algorithm can efficiently deal the tensor models with the larger size, density, and rank. 

\textbf{Fast Tensor Factorization Algorithms}: Factorization algorithms namely Alternating Least Square (ALS) [23], Multiplicative Update rule (MU) \cite{cichocki2009nonnegative}, Gradient Descent (GD) \cite{zhang2016fast} and Coordinate Descent (CD) \cite{hsieh2011fast,wright2015coordinate,nesterov2012efficiency} have been commonly used to factorize the tensor models. These traditional algorithms have been extended to tensors based on MF. Researchers have explored the concepts of gradient calculation and tensor block to fasten the factorization process in these factorization algorithms. Fast Hierarchical ALS (FHALS) \cite{phan2013fast} and Fast MU (FMU) \cite{phan2012fast} are the optimized variants of ALS and MU by exploiting an efficient gradient calculation method for the faster execution. 
The traditional gradient calculation requires the tensor unfolding at each iteration which is relatively slow. FHALS and FMU use a simplified update rule with the gradient calculation that does not require the tensor unfolding. Accelerated Proximal Gradient (APG) \cite{zhang2016fast} utilizes a low-rank approximation to speed up the factorization process. The Block Coordinate Descent (BCDP)~\cite{xu2013block} technique has been applied in NTF to minimize the complexity by processing the tensor in convex blocks. The element selection-based CD algorithm called Greedy Coordinate Decent (GCD) \cite{hsieh2011fast} has shown fast convergence in Nonnegative Matrix Factorization (NMF) process. But the extension of GCD to NTF involves high complexity because of the frequent gradient updates \cite{balasubramaniam2018nonnegative}. 

Few researchers have explored parallel and distributed computational methods to minimize the time complexity involved in the Matricized Tensor Times khatri-Rao Product ($mttkrp$) and matrics updates during TF~\cite{shin2017fully,oh2017s,kang2012gigatensor,choi2014dfacto,beutel2014flexifact,park2016bigtensor,smith2015splatt}. CCD++~\cite{ccdpp} is a CD algorithm for MF that updates the elements in a factor matrix column-wise. Coordinate Descent for Tensor Factorization (CDTF)~\cite{shin2017fully} and Subset Alternating Least Square (SALS)~\cite{shin2014distributed} are two recent extensions of CCD++ for TF. While CDTF is a direct extension of CCD++ for TF, SALS is a special case with an additional constraint that controls the number of columns to update in a single iteration. The column-wise factor matrix update minimizes the complexity of $mttkrp$ in these algorithms and introduces a huge advantage in the distributed and parallel environment. However, these implementations do not alter the computational complexity inherent in the underlying factorization algorithm. We have implemented CDTF for NTF and used it as a benchmark in experiments. SALS cannot be directly extended for NTF due to the additional constraint it imposes. 

\textbf{Limitations}: While NTF has been applied widely with tensor based methods to utilize multi-dimensional data, it suffers from the scalability issue due to the complex matrix and tensor products involved in the factorization process \cite{shin2017fully}. Existing factorization algorithms are limited to small datasets only (as listed in Table~\ref{tab:one}). While parallel and distributed computational methods improve the runtime and scalability by utilizing huge resources, performance on a normal machine with fewer memory requirements remains poor \cite{oh2017s}. While optimizing the basic factorization process can improve the performance without depending on huge resources, not all the optimized factorization algorithms are easily extendable to a parallel and distributed environment to further improve the performance \cite{kimura2015variable}. Most importantly, not all the factorization algorithms can effectively handle tensors of large size, sparsity, and rank \cite{oh2017s}. We propose a factorization algorithm, SaCD, to efficiently process the large and sparse tensor model, with and without parallelization, by avoiding the frequent gradient calculations that are essential in existing algorithms \cite{hsieh2011fast,balasubramaniam2018nonnegative}.

\section{Basics: Preliminaries and Definitions}
The notations used in this paper are summarized in Table~\ref{table_notations}. The process of converting a tensor into a matrix is called as matricization or unfolding of tensor \cite{tucker1966some}. The $mode-1$ matricization can be denoted as \(\boldsymbol{\mathrm{X_1}} \in \R^{(Q \times (PS))}\) for a third order tensor \( \boldsymbol{\mathcal{X}} \in \R^{(Q \times P \times S)}\) where $Q$, $P$, and $S$ denotes the mode length. Several matrix products are required for tensor-based processes \cite{lathauwer2008decompositions}. We briefly introduce them.

\begin{table}[!t]
\caption{Table of Symbols}
\label{table_notations}
\centering
\begin{tabular}{ll}
\toprule
\bfseries Symbol & \bfseries Definition\\
\toprule

\( \boldsymbol{\mathcal{X}} \)   & tensor (Euler script letter)\\
$U,V,W$ & 3 modes of \( \boldsymbol{\mathcal{X}}\) \\
$Q,P,S$ & Length of mode $U,V,W$ respectively \\
\( \Omega \)   & set of indices of observable entries of \( \boldsymbol{\mathcal{X}} \)\\
\( x_{qps} \)   & $(q,p,s)^{th}$ entry of \( \boldsymbol{\mathcal{X}} \)\\
\( \Omega^U_q \)   & subset of $\Omega$ whose mode $U$'s index is $q$ \\
\( \Omega^V_p \)   & subset of $\Omega$ whose mode $V$'s index is $p$ \\
\( \Omega^W_s \)   & subset of $\Omega$ whose mode $W$'s index is $s$ \\
\( \boldsymbol{\mathrm{U}}\) & matrix (upper case. bold letter)\\
\( \boldsymbol{\mathrm{u}}\) & vector (lower case, bold letter)\\
\( \mathit{u}\) & scalar (lower case, italic letter) / element\\
\( \boldsymbol{\mathrm{X_n}}\) & mode-n matricization of tensor\\
\( R\) & rank of tensor\\
\(\otimes\) & Kronecher product\\
\(\odot\) & Khatri-Rao product\\
\(\ast\) & Hadamard product\\
\(\circ\) & outer product\\
\(\norm{.}\) & Frobenius norm\\
\bottomrule
\end{tabular}
\end{table}

\subsection{Kronecker Product}
For two matrices denoted as \(\boldsymbol{\mathrm{U}} \in \R^{(Q \times R)}\) and \(\boldsymbol{\mathrm{V}} \in \R^{(P \times R)}\), the Kronecker product is presented as \(\boldsymbol{\mathrm{U}} \otimes \boldsymbol{\mathrm{V}}\). The resultant matrix of size \( (QP  \times R^2) \) is defined as follows:

\begin{equation}
\label{eq_1}
\boldsymbol{\mathrm{U}} \otimes  \boldsymbol{\mathrm{V}} = \begin{bmatrix}
    u_{11}\boldsymbol{\mathrm{V}} & u_{12}\boldsymbol{\mathrm{V}} & u_{13}\boldsymbol{\mathrm{V}} & \dots  & u_{1r}\boldsymbol{\mathrm{V}} \\
    u_{21}\boldsymbol{\mathrm{V}} & u_{22}\boldsymbol{\mathrm{V}} & u_{23}\boldsymbol{\mathrm{V}} & \dots  & u_{2r}\boldsymbol{\mathrm{V}} \\
    \vdots & \vdots & \vdots & \ddots & \vdots \\
    u_{q1}\boldsymbol{\mathrm{V}} & u_{q2}\boldsymbol{\mathrm{V}} & u_{q3}\boldsymbol{\mathrm{V}} & \dots  & u_{qr}\boldsymbol{\mathrm{V}}
\end{bmatrix}
\end{equation}

\begin{equation}
  \label{eq_2} 
  = \begin{bmatrix} 
  \boldsymbol{\mathrm{u_{1}}} \boldsymbol{\mathrm{V}} & \boldsymbol{\mathrm{u_{2}}}\boldsymbol{\mathrm{V}} & \boldsymbol{\mathrm{u_{3}}}\boldsymbol{\mathrm{V}} & \dots  & \boldsymbol{\mathrm{u_{r}}}\boldsymbol{\mathrm{V}} \\
  \end{bmatrix},  
\end{equation}
where \(\boldsymbol{\mathrm{u_{r}}}\) is $r^{th}$ column of the factor matrix \(\boldsymbol{\mathrm{U}}\).

\subsection{Khatri-Rao Product}
The column-wise Kronecker product, called as Khatri-Rao product, is denoted as \(\boldsymbol{\mathrm{U}} \odot \boldsymbol{\mathrm{V}}\). The resultant matrix of size \( (QP  \times R)\) is defined as:
\begin{equation}
  \label{eq_3} 
 \boldsymbol{\mathrm{U}} \odot  \boldsymbol{\mathrm{V}}  = \begin{bmatrix} 
  \boldsymbol{\mathrm{u_1}} \otimes \boldsymbol{\mathrm{v_1}} &  \dots  & \boldsymbol{\mathrm{u_r}} \otimes \boldsymbol{\mathrm{v_r}} \\
  \end{bmatrix},  
\end{equation}
where \(\boldsymbol{\mathrm{u_{r}}}\) and \(\boldsymbol{\mathrm{v_{r}}}\) are columns of the matrices \(\boldsymbol{\mathrm{U}}\) and \(\boldsymbol{\mathrm{V}}\) respectively.

\subsection{Hadamard Product}
When the size of two matrices are same, the hadamard product can be calculated by the element-wise matrix product. It is denoted by \(\boldsymbol{\mathrm{U}} \ast \boldsymbol{\mathrm{V}}\) and defined as: 

\begin{equation}
\label{eq_4}
\boldsymbol{\mathrm{U}} \ast  \boldsymbol{\mathrm{V}} = \begin{bmatrix}
    u_{11}v_{11} & u_{12}v_{12} & u_{13}v_{13} & \dots  & u_{1r}v_{1r} \\
   u_{21}v_{21} & u_{22}v_{22} & u_{23}v_{23} & \dots  & u_{2r}v_{2r} \\
    \vdots & \vdots & \vdots & \ddots & \vdots \\
    u_{q1}v_{q1} & u_{q2}v_{q2} & u_{q3}v_{q3} & \dots  & u_{qr}v_{qr}
\end{bmatrix}.
\end{equation}

\subsection{Tensor Factorization}
Tensor factorization, a dimensionality reduction technique, is an extension of matrix factorization for higher order. It factorizes a tensor into factor matrices that contain latent features. CANDECOMP/PARAFAC (CP) and Tucker are two well-known factorization techniques \cite{kolda2009tensor}. CP has shown to be less expensive in both memory and time as compared to Tucker \cite{oh2017s,kolda2009tensor}.  

\textbf{\textit{Definition 1} (CP Factorization): } For a tensor \(\boldsymbol{\mathcal{X}} \in \R^{(Q\times P\times S)} \) and rank \(R\), CP factorization factorizes the tensor into a sum of component rank-one tensors \cite{carroll1970analysis} (as shown in Figure~\ref{fig_sim}) as:
\begin{equation}
\label{eq_5}
   \boldsymbol{\mathcal{X}} \cong \llbracket\boldsymbol{\mathrm{U}}, \boldsymbol{\mathrm{V}}, \boldsymbol{\mathrm{W}} \rrbracket = \sum_{r=1}^{R} \boldsymbol{\mathrm{u_{r}}} \circ \boldsymbol{\mathrm{v_{r}}} \circ \boldsymbol{\mathrm{w_{r}}},
\end{equation}
where \(\boldsymbol{\mathrm{U}} \in \R^{(Q \times R)}\), \(\boldsymbol{\mathrm{V}} \in \R^{(P \times R)}\) and \(\boldsymbol{\mathrm{W}} \in \R^{(S \times R)}\) are factor matrices with \(R\) hidden features, \(R \in \mathbb{Z}_{+} \). 
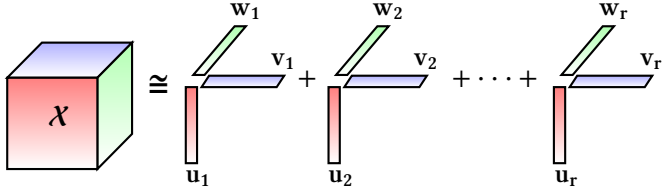
\begin{figure}[!t]
\centering
\begin{tikzpicture}
        \tikzset{xzplane/.style={canvas is xz plane at y=#1, thick,shading=axis,shading angle=90,top color = blue!30,draw=black}};
        \tikzset{yzplane/.style={canvas is yz plane at x=#1,thick,shading=axis,shading angle=90,top color = green!30,draw=black}};
        \tikzset{xyplane/.style={canvas is xy plane at z=#1,thick,shading=axis,shading angle=90,top color = red!50,draw=black}};
        \tikzset{xyplane2/.style={canvas is xy plane at z=#1,thick,shading=axis,shading angle=90,top color = red!50,draw=black}};

        \draw[xzplane=1.2] (0,0)--(1.2,0)--(1.2,1.2)--(0,1.2) --cycle;
        \draw[yzplane=1.2] (0,0)--(1.2,0)--(1.2,1.2)--(0,1.2) --cycle;
        \draw[xyplane=1.2] (0,0)--(1.2,0)--(1.2,1.2)--(0,1.2) --cycle;      
        \draw[xyplane=2] node[black] at (1,1){\(\boldsymbol{\mathcal{X}}\)};
        \begin{scope}[shift={(1.9 cm,0 cm)}]
            \draw[xyplane=2.2] node[black] at (0.5,1.5){\Large\(\boldsymbol{\cong}\)};
        \end{scope}
        
        \begin{scope}[shift={(2.3 cm,0 cm)}]
            \draw[xyplane=1] (0,0)--(0.15,0)--(0.15,1)--(0,1) --cycle;
            \draw[xyplane=1, top color = blue!30] (0.21,1)--(1.21,1)--(1.31,1.15)--(0.31,1.15) --cycle;
            \draw[xyplane=1, top color = green!30] (0.1,1.16)--(0.25,1.16)--(0.80,1.8)--(0.65,1.8) --cycle;
            \draw[xyplane=1] node[black] at (0.17,-0.2){\(\boldsymbol{\mathrm{u_1}}\)};
            \draw[xyplane=1] node[black] at (1.3,1.34){\(\boldsymbol{\mathrm{v_1}}\)};
            \draw[xyplane=1] node[black] at (0.8,2){\(\boldsymbol{\mathrm{w_1}}\)};
        \end{scope}
        
        \begin{scope}[shift={(3.8 cm,0 cm)}]
            \draw[xyplane=2] node[black] at (0.5,1.5){\Large\(+\)};
        \end{scope}
        \begin{scope}[shift={(4.2 cm,0 cm)}]
            \draw[xyplane=1] (0,0)--(0.15,0)--(0.15,1)--(0,1) --cycle;
            \draw[xyplane=1, top color = blue!30] (0.21,1)--(1.21,1)--(1.31,1.15)--(0.31,1.15) --cycle;
            \draw[xyplane=1, top color = green!30] (0.1,1.16)--(0.25,1.16)--(0.80,1.8)--(0.65,1.8) --cycle;
            \draw[xyplane=1] node[black] at (0.17,-0.2){\(\boldsymbol{\mathrm{u_2}}\)};
            \draw[xyplane=1] node[black] at (1.3,1.34){\(\boldsymbol{\mathrm{v_2}}\)};
            \draw[xyplane=1] node[black] at (0.8,2){\(\boldsymbol{\mathrm{w_2}}\)};
        \end{scope}
        
        \begin{scope}[shift={(6.3 cm,0 cm)}]
            \draw[xyplane=2] node[black] at (0.5,1.5){\Large\(+ \dots +\)};
        \end{scope}
         \begin{scope}[shift={(7.2 cm,0 cm)}]
            \draw[xyplane=1] (0,0)--(0.15,0)--(0.15,1)--(0,1) --cycle;
            \draw[xyplane=1, top color = blue!30] (0.21,1)--(1.21,1)--(1.31,1.15)--(0.31,1.15) --cycle;
            \draw[xyplane=1, top color = green!30] (0.1,1.16)--(0.25,1.16)--(0.80,1.8)--(0.65,1.8) --cycle;
            \draw[xyplane=1] node[black] at (0.17,-0.2){\(\boldsymbol{\mathrm{u_r}}\)};
            \draw[xyplane=1] node[black] at (1.3,1.34){\(\boldsymbol{\mathrm{v_r}}\)};
            \draw[xyplane=1] node[black] at (0.8,2){\(\boldsymbol{\mathrm{w_r}}\)};
        \end{scope}
        
    \end{tikzpicture}
\caption{CP Factorization.}
\label{fig_sim}
\end{figure}

The goal of CP tensor factorization is to approximate the input tensor and its objective function can be formulated as follows, 

\begin{figure}[!t]%
	\centering
	\includegraphics[width=13cm,height=8cm]{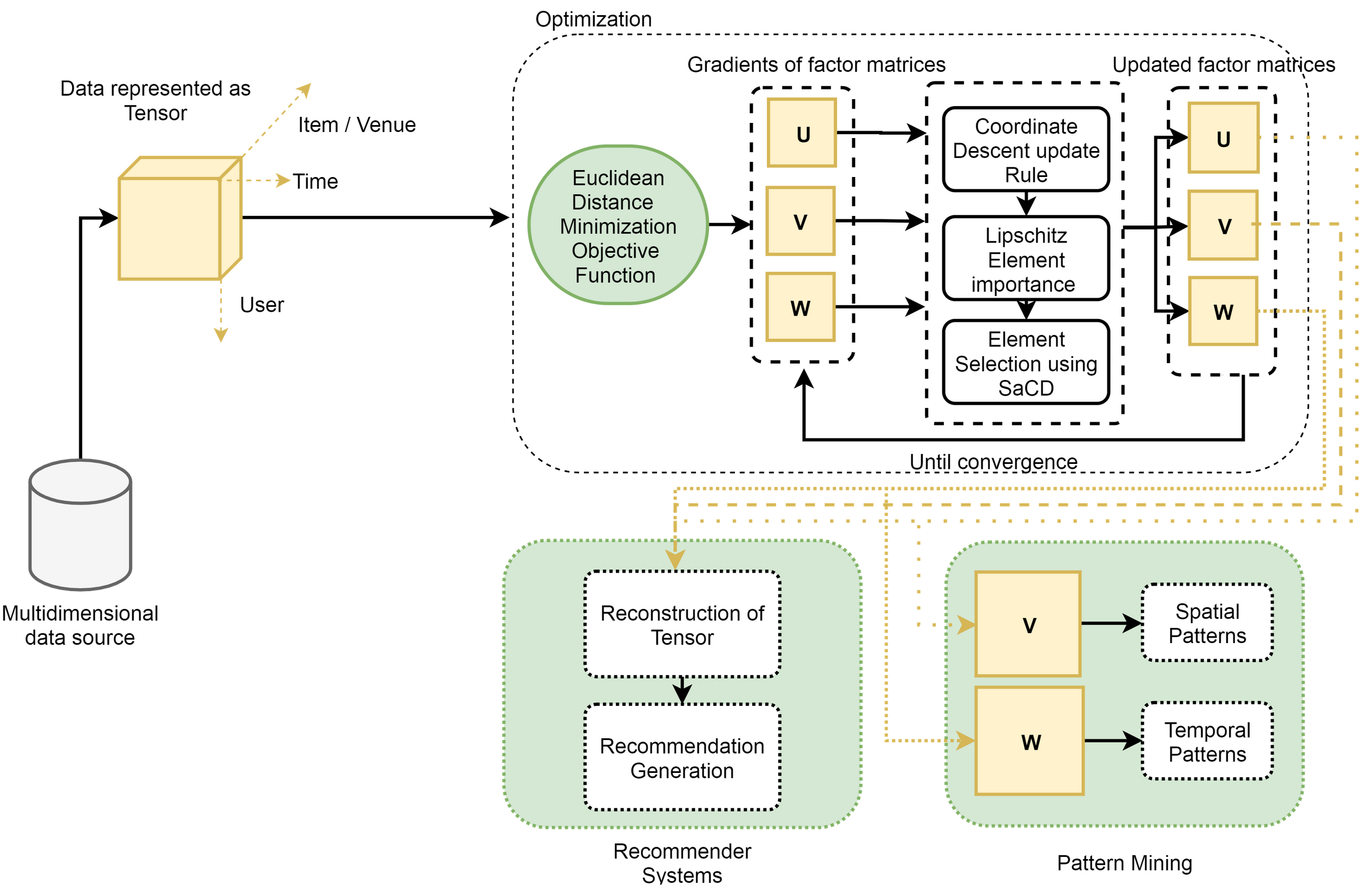}
	\caption{Architecture of the overall process}
	\label{arch}
\end{figure}
\begin{equation}
\label{eq_6}
\min_{\boldsymbol{\mathrm{U}}, \boldsymbol{\mathrm{V}}, \boldsymbol{\mathrm{W}}}f(\boldsymbol{\mathrm{U}}, \boldsymbol{\mathrm{V}}, \boldsymbol{\mathrm{W}}) = \norm{\boldsymbol{\mathcal{X}} - \llbracket\boldsymbol{\mathrm{U}}, \boldsymbol{\mathrm{V}}, \boldsymbol{\mathrm{W}} \rrbracket }^2.
\end{equation}

The learning process is modeled as the minimization of Euclidean loss function as shown in Equation~\eqref{eq_6}. The optimization process (Equation~\eqref{eq_7}) aims to learn the factors to reduce the difference between the original tensor and approximated tensor, where the predicted value \( \hat{x}_{qps}\) is calculated by the inner product of learned factors across all the features, as formulated in Equation~\eqref{eq_8} \cite{cichocki2009nonnegative}:
\begin{equation}
\label{eq_7}
f(\boldsymbol{\mathrm{U}}, \boldsymbol{\mathrm{V}}, \boldsymbol{\mathrm{W}}) = \sum_{q=1}^Q\sum_{p=1}^P\sum_{s=1}^S \norm{x_{qps}-\hat{x}_{qps}}^2.
\end{equation}
\begin{equation}
\label{eq_8}
\hat{x}_{qps} = \sum_{r=1}^R u_{qr}.v_{pr}.w_{sr}.
\end{equation}

\subsection{Nonnegative Tensor Factorization (NTF)}
The basic idea of NTF is factorizing the $n$-dimensional tensor into $n$ factor matrices that satisfy the nonnegative constraint~\cite{kolda2009tensor}. NTF can be achieved in traditional Tucker or CP factorization model by imposing constraints to maintain the nonnegative values. 

The objective function for NTF based on CP model can be formulated as follows:
\begin{equation}
\label{eq:ntf}
\min_{\boldsymbol{\mathrm{U}}, \boldsymbol{\mathrm{V}}, \boldsymbol{\mathrm{W}} \geq{0}}f(\boldsymbol{\mathrm{U}},\boldsymbol{\mathrm{V}},\boldsymbol{\mathrm{W}})  = \norm{\boldsymbol{\mathcal{X}}-\llbracket \boldsymbol{\mathrm{U}},\boldsymbol{\mathrm{V}},\boldsymbol{\mathrm{W}} \rrbracket}^2. 
\end{equation}

The approximated tensor, after factorization, using the learned factor matrices will be a denser model in which a large portion of absent values are populated \cite{nguyen2016fast,sra2008non}. For applications such as predictive modelling and recommender systems, these populated values can be inferred as potential prediction or recommendation \cite{ifada2014tensor,balasubramaniam2018people}.

The goal of the optimization problem in Equation~\eqref{eq:ntf} is to find the accurate factor matrices \(\boldsymbol{\mathrm{U}}\), \(\boldsymbol{\mathrm{V}}\) and \(\boldsymbol{\mathrm{W}}\). ALS~\cite{takane1977nonmetric} is the most common algorithm used to find the factor matrices~\cite{sidiropoulos2017tensor,phan2013fast}. It gives equal importance to all the elements and alternatively updates the entire factor matrix by fixing all the other factor matrices. This unnecessary element updates can cause slow convergence and poor scalability \cite{acar2011all,hsieh2011fast}. 

\section{Saturating Coordinate Descent Algorithm for NTF}

In this paper, we propose the Saturating Coordinate Descent (SaCD) algorithm to solve Equation~\eqref{eq:ntf} in order to reduce the complexity inherent in the factor matrix update and to overcome the poor scalability faced by traditional algorithms like ALS~\cite{phan2013fast}, APG~\cite{zhang2016fast}, GCD~\cite{hsieh2011fast}, CDTF~\cite{shin2017fully} and BCDP~\cite{xu2013block}. Figure~\ref{arch} provides the detail of the overall process. 

\subsection{Factor Matrix Update Rule}
We First explain the learning process of factor matrix \(\boldsymbol{\mathrm{U}}\) which is applicable to other factor matrices. The factor matrix \(\boldsymbol{\mathrm{U}}\) is updated by solving Equation~\eqref{eq_ucap} and to solve, we need to find the first order and second partial derivatives (Equation~\eqref{eq_9} and Equation~\eqref{eq_10} respectively) for \(f\) in Equation~\eqref{eq:ntf} as follows,

\begin{equation}
\label{eq_ucap}
 \min_{\boldsymbol{\mathrm{U}} \geq{0}}  \norm{\boldsymbol{\mathrm{X_1}} - \boldsymbol{\mathrm{U}}(\boldsymbol{\mathrm{W}} \odot \boldsymbol{\mathrm{V}})^T}^2,
\end{equation}
	
\begin{equation}
\label{eq_9}
\frac{\partial f}{\partial \boldsymbol{\mathrm{U}}} = \boldsymbol{\mathrm{G}} =  -\boldsymbol{\mathrm{X_1}}(\boldsymbol{\mathrm{W}} \odot \boldsymbol{\mathrm{V}}) + \boldsymbol{\mathrm{U}}(\boldsymbol{\mathrm{V}}^T\boldsymbol{\mathrm{V}} \ast \boldsymbol{\mathrm{W}}^T\boldsymbol{\mathrm{W}}),
\end{equation}

\begin{equation}
\label{eq_10}
\frac{\partial^2 f}{\partial \boldsymbol{\mathrm{U}}} = \boldsymbol{\mathrm{H}} =  \boldsymbol{\mathrm{V}}^T\boldsymbol{\mathrm{V}} \ast \boldsymbol{\mathrm{W}}^T\boldsymbol{\mathrm{W}},
\end{equation}
where $\frac{\partial f}{\partial \boldsymbol{\mathrm{U}}}$ and $\frac{\partial^2 f}{\partial \boldsymbol{\mathrm{U}}}$ are the first order and second order partial derivatives of the objective function $f$ (Equation~\eqref{eq:ntf}) with respect to \(\boldsymbol{\mathrm{U}}\) and let us denote them as $\boldsymbol{\mathrm{G}}$ (gradient) and $\boldsymbol{\mathrm{H}}$ (second order derivative) respectively.

The one variable gradient and second order derivative are computed as,
\begin{equation}
\label{eq_14}
g_{qr} = -(\boldsymbol{\mathrm{X_1}}(\boldsymbol{\mathrm{W}} \odot \boldsymbol{\mathrm{V}}))_{qr} + (\boldsymbol{\mathrm{U}}(\boldsymbol{\mathrm{V}}^T\boldsymbol{\mathrm{V}} \ast \boldsymbol{\mathrm{W}}^T\boldsymbol{\mathrm{W}}))_{qr},
\end{equation}
\begin{equation}
\label{eq_15}
h_{rr} = (\boldsymbol{\mathrm{V}}^T\boldsymbol{\mathrm{V}} \ast \boldsymbol{\mathrm{W}}^T\boldsymbol{\mathrm{W}})_{rr},
\end{equation}
where \(g_{qr}\) and \(h_{rr}\) represents \(q,r^{th}\) and \(r,r^{th}\) gradient \(\boldsymbol{\mathrm{G}}\) and second order derivative \(\boldsymbol{\mathrm{H}}\) respectively. 

Equation~\eqref{eq:ntf} becomes a quadratic equation in terms of the updated parameter when all other parameters are fixed. This leads to the following closed form CD update rule (i.e., one variable sub-problem) \cite{phan2013fast} for each parameter:

\begin{equation}
\label{eq_11}
\hat{u}_{qr} = \frac{g_{qr}}{h_{rr}}.
\end{equation}
The one variable sub-problem can be a simplified ALS update rule where each factor matrix is updated element-wise, thus it reduces the computational complexity. The nonnegative constraint can be added to Equation~\eqref{eq_11} as,
\begin{equation}
\label{eq_12}
\hat{u}_{qr} \xleftarrow{} \max(0,u_{qr} - \hat{u}_{qr})-u_{qr},
\end{equation}
where \(u_{qr}\)  indicates the \(q,r^{th}\) element of the factor matrix \(\boldsymbol{\mathrm{U}}\) and \(\hat{u}_{qr}\) indicates the computed element.

With the computed element value ($\hat{u}_{qr}$), the element \(u_{qr}\) can be updated as,
\begin{equation}
\label{eq_13}
u_{qr} \xleftarrow{} u_{qr} + \hat{u}_{qr}.
\end{equation}

Since the calculation of \(\boldsymbol{\mathrm{X_1}}(\boldsymbol{\mathrm{W}} \odot \boldsymbol{\mathrm{V}})\) and  \(\boldsymbol{\mathrm{V}}^T\boldsymbol{\mathrm{V}} \ast \boldsymbol{\mathrm{W}}^T\boldsymbol{\mathrm{W}}\) for every element is expensive, it is better to calculate this value only once at each iteration, instead of calculating it for every element update, during the factorization process to find the best learning factor matrices. 

In NMF, element selection has been proven to converge faster for updating the important elements repeatedly, instead of considering all elements \cite{hsieh2011fast}. The traditional measurement of element importance and update is computationally inefficient for NTF due to frequent gradient updates.  Next, we show how an element importance can be efficiently calculated and only important elements can be updated that can avoid frequent gradient updates.

\subsection{Proposed Lipschitz Element Importance}
The existing partial differential equations based optimization algorithms like APG, FHALS, FMU, and GCD are prone to slow convergence due to inconsistent gradients, i.e., gradients shift on both directions from the global or local minima in the non-convex optimization curve~\cite{khamaru2018convergence}. Additional regularity conditions has been incorporated to speed-up the convergence in NMF~\cite{guan2012nenmf}. In this paper, we propose to smooth (Lipschitz smoothness) the continuous function $f$ with a strong condition of Lipschitz continuity for fast convergence. 

\textbf{\textit{Definition 2} (Lipschitz continuity): }
Lipschitz continuity is a strong form of uniform continuity for a function if the function \(f\) is differentiable with Lipschitz constant \(L\).

The continuously differentiable function $f$ is called Lipschitz continuous with Lipschitz constant $L$, such that:
\begin{equation}
\label{eq_21}
\norm{f(u^k)-f(u^{k+1})}\leq L\norm{u^k -u^{k+1}}, \forall u^k \in \boldsymbol{\mathrm{U}}^k, u^{k+1} \in \boldsymbol{\mathrm{U}}^{k+1},
\end{equation}
where \(\boldsymbol{\mathrm{U}}^k\) and \(\boldsymbol{\mathrm{U}}^{k+1} \) are two factor matrices at two consecutive iterations.

\textbf{\textit{Definition 3} (Lipschitz smoothness): }
The Lipschitz continuous function, as defined in  Equation~\eqref{eq_21}, can be smoothed with the upper bound value for $L$. It ensures the strong convexity of the function and achieves faster convergence. This property is called Lipschitz smoothness and the function \(f\) with this property is called $L$-smoothed. It can be defined as,

\begin{equation}
\label{eq_lpzsmooth}
f(u^k) \geq f(u^{k+1}) + f'(u^{k+1})(u^k - u^{k+1}) + \frac{L}{2}\norm{f(u^k)-f(u^{k+1})}.
\end{equation}

The upper bound value for $L$ should be calculated to satisfy,
\begin{equation}
\label{eq_lpzsmoothsatisfy}
f''(\boldsymbol{\mathrm{U,V,W}}) - L\boldsymbol{\mathrm{I}} \geq{0},
\end{equation}
where $\boldsymbol{\mathrm{I}}$ is an Identity matrix.

Equations~\eqref{eq_9} and \eqref{eq_10} can reveal that the objective function \(f\) is differentiable with respect to factor matrices as well as it is non-convex. We propose to analyse these properties using Lipschitz continuity and Lipschitz smoothness and measure the element importance in order to achieve faster convergence.

Let \(\boldsymbol{\mathrm{Z}}\) be the element importance matrix that holds the importance value for each element in a factor matrix. The importance of element \(u_{qr}\) can be calculated as the difference in the objective function as,

\begin{equation}
\label{eq_16}
z_{qr} = g_{qr}-\hat{g}_{qr},
\end{equation}
where \(g_{qr}\) represents the gradient of \(u_{qr}\) and \(\hat{g}_{qr}\) represents the new gradient value of the computed element \(\hat{u}_{qr}\).

With the precomputed \(g_{qr}\) as per Equation~\eqref{eq_14}, we can compute \(\hat{g}_{qr}\) using one variable subproblem \cite{hsieh2011fast} as follows,
\begin{equation}
\label{eq_17}
\hat{g}_{qr} = g_{qr} + g_{qr}\hat{u}_{qr} + \frac{1}{2}(h_{rr} \hat{u}^2_{qr}).
\end{equation}

Substituting Equations~\eqref{eq_14} and \eqref{eq_17} in Equation~\eqref{eq_16}, we get the element importance as,
\begin{equation}
\label{eq_18}
z_{qr} = -(g_{qr}\hat{u}_{qr}) -0.5(h_{rr}\hat{u}^2_{qr}).
\end{equation}



In Equation~\eqref{eq_18}, \(h_{rr}\)  is the partial derivative of gradient \(g_{qr}\) as defined in Equation~\eqref{eq_15}. An objective function holding Lipschitz continuity with differentiable gradients can be Lipschitz-smoothed ($L$-smoothed) (as proven in Lemma \ref{lemma1}). This property will allow the function to converge faster by replacing the partial derivative of gradient (i.e. \(h_{rr}\) in~\eqref{eq_18}) with Lipschitz constant $L$ \cite{guan2012nenmf}. 

We propose to calculate the difference in the $L$-smoothed objective function for each element’s update during the factorization using Lipschitz continuity \cite{hager1979lipschitz} that can achieve faster convergence. The convergence of an optimization problem can be analyzed using curvature measure \(C_f\) that measures the deviation of the objective function $f$ with the linear approximation. The Lipschitz smoothness of $f$ has shown the best \(C_f\) with the improvement of convergence speed to \(\frac{1}{K^2}\) where \(K\) is the number of iterations~\cite{jaggi2011sparse} \footnote{For a formal analysis of curvature measure and Lipschitz continuity, we direct the readers to \cite{jaggi2011sparse}.}. We conjecture that use of Lipschitz continuity in the element importance calculation will speed up the convergence.

By applying the upper bound value of Lipschitz continuity in the one variable subproblem Equation~\eqref{eq_17}, we have:

\begin{equation}
\label{eq_19}
\hat{g}_{qr} = g_{qr} + g_{qr}\hat{u}_{qr}+\frac{L}{2}(\hat{u}^2_{qr}),
\end{equation}
where \(L\) is a positive scalar called Lipschitz constant.

With the redefined one variable subproblem, we can define Lipschitz element importance by substituting Equations~\eqref{eq_14} and \eqref{eq_19} in Equation~\eqref{eq_16} as,
\begin{equation}
\label{eq_20}
z_{qr} = -(g_{qr} \hat{u}_{qr}) -\frac{L}{2}(\hat{u}^2_{qr}).
\end{equation}

\begin{lemma}
\label{lemma1}
The gradient of the objective function Equation~\eqref{eq:ntf} satisfies Lipschitz continuity.
\end{lemma}

\begin{proof}
Supposedly we have two factor matrices \(\boldsymbol{\mathrm{U}}^k\) and \(\boldsymbol{\mathrm{U}}^{k+1} \) at two consecutive iterations.

\begin{multline}
\label{eq_22}
\norm{\frac{\partial f}{\partial \boldsymbol{\mathrm{U}}^{k-1}} - \frac{\partial f}{\partial \boldsymbol{\mathrm{U}}^{k}}} = \norm{-\boldsymbol{\mathrm{X_1}}(\boldsymbol{\mathrm{W}}\odot\boldsymbol{\mathrm{V}})+\boldsymbol{\mathrm{U}}^{k-1}(\boldsymbol{\mathrm{V}}^T\boldsymbol{\mathrm{V}} \ast \boldsymbol{\mathrm{W}}^T\boldsymbol{\mathrm{W}}) 
-(-\boldsymbol{\mathrm{X_1}}(\boldsymbol{\mathrm{W}}\odot\boldsymbol{\mathrm{V}})+\boldsymbol{\mathrm{U}}^{k}(\boldsymbol{\mathrm{V}}^T\boldsymbol{\mathrm{V}} \ast \boldsymbol{\mathrm{W}}^T\boldsymbol{\mathrm{W}}))}.
\end{multline}

Applying \textit{Definition 2} and assuming the objective function is differentiable, we obtain,

\begin{equation}
\label{eq_23}
\norm{\frac{\partial f}{\partial \boldsymbol{\mathrm{U}}^{k-1}} - \frac{\partial f}{\partial \boldsymbol{\mathrm{U}}^{k}}} = L \norm{\boldsymbol{\mathrm{U}}^{k-1} - \boldsymbol{\mathrm{U}}^k}.
\end{equation}

Equating Equations~\eqref{eq_22} and \eqref{eq_23}, we can identify the value for \(L\) as \(\norm{\boldsymbol{\mathrm{V}}^T\boldsymbol{\mathrm{V}} \ast \boldsymbol{\mathrm{W}}^T\boldsymbol{\mathrm{W}}}\). \(L\) is the singular upper bound value that defines the maximum curves an objective function allowed to have and making the function \(f\) \(L\)-smoothed for faster convergence.

Solving Equation~\eqref{eq_22} and Equation~\eqref{eq_23} for one variable subproblem with \(u^k \in \boldsymbol{\mathrm{U}}^k\) and \(u^{k+1} \in \boldsymbol{\mathrm{U}}^{k+1}\)  proves lemma \ref{lemma1}.
\end{proof}

\subsection{SaCD (Saturating Coordinate Descent)}
Once we have the element importance calculated for all the elements in the factor matrix $\boldsymbol{\mathrm{U}}$ using Equation~\eqref{eq_20} and stored them in $\boldsymbol{\mathrm{Z}}$, we now select a set of elements for every iteration using the proposed SaCD algorithm. The state-of-the-art GCD NMF algorithm uses greedy strategy in finding and updating single element multiple times in each row~\cite{hsieh2011fast}. This requires a very expensive frequent gradient update. In the minimization optimization problem as formulated in Equation~\eqref{eq:ntf}, the error is minimized for each iteration. It slowly reaches convergence, or reaches a saturation point beyond which updating will not minimize the objective function $f$. Hence, the contribution of each element in minimizing the objective function decreases with each iteration. Moreover, not all the elements will effectively minimize the objective function. It is sufficient and efficient to update every single element until it reaches the saturation point, instead of updating it for all the iterations. This avoids the expensive frequent gradient update as proven in Lemma \ref{lemma2}.

As the element importance \(z_{qr}\) is the difference in the objective function as per Equation~\eqref{eq_16}, we can identify the saturation point \(sp_{qr}\) of each element for every iteration $k$, by keeping track of the previous value of \(z_{qr}^{k-1}\) as, 
\begin{equation}
\label{eq_24}
sp_{qr} = (z^k_{qr}-z^{k-1}_{qr}).
\end{equation}

Additionally, for each iteration, we measure the total importance of a factor matrix as,
\begin{equation}
\label{eq_25}
ti^k = \sum_{q=1}^{Q}\sum_{r=1}^{R}z^k_{qr},
\end{equation}
\begin{equation}
\label{eq_26}
ti^{k-1} = \sum_{q=1}^{Q}\sum_{r=1}^{R}z^{k-1}_{qr}.
\end{equation}
If the total importance of current iteration \(ti^k\)  is more than the previous iteration \(ti^{(k-1)}\), we identify the saturation point \(sp_{qr}\) as,
\begin{equation}
\label{eq_27}
sp_{qr} = (z_{qr}^{k-1}-z_{qr}^k).
\end{equation}

While the difference in the objective function gradually decreases, sometimes it increases. This happens when the optimization is reaching local or global minima. To avoid stuck in local minima, this redefinition of the saturation point is needed that further allows important elements to be updated until global minima is reached. 
We use this saturation point to decide if the element \(u_{qr}\) is to be updated as per Equation~\eqref{eq_13}. If \(sp_{qr} < 0\), we avoid updating that element. 

We have explained the proposed algorithm by describing the learning process of factor matrix \(\boldsymbol{\mathrm{U}}\). Next we briefly show that the process is applicable to other factor matrices too.

\begin{lemma}
	\label{lemma2}
	In the SaCD element selection, for each update of \(u_{qr}\), it is not necessary to update gradient of all the columns of \(q^{th}\) row and element importance \(z_{qr}\). 
\end{lemma}
\begin{proof}
	Let $u_{qr} \in \boldsymbol{\mathrm{U}}$ be $q,r^{th}$ element of $\boldsymbol{\mathrm{U}}$ and $g_{qr} \in \boldsymbol{\mathrm{G}}$ represents $q,r^{th}$ gradient of $\boldsymbol{\mathrm{U}}$.
	
	For every $r:R$, the set of important elements $\boldsymbol{e_r}$  are selected based on the saturation point calculated using Equation~\eqref{eq_24} or Equation~\eqref{eq_27}. 
	
	The subset $\boldsymbol{e_r}$ is dependent on $r^{th}$ column of $\boldsymbol{\mathrm{G}}$. 
	
	For each update of $u_{qr} \in \boldsymbol{e_r}, g_{qr} \in \boldsymbol{\mathrm{G}}$ alone needs to be updated as $\boldsymbol{e_r} \independent{} \boldsymbol{\mathrm{g}}_{y}$ where $y\neq r$ and $g_{qr} \independent{} u_{nr}$ where $n\neq q$ and $\independent{}$ indicates that $e_r$ is independent of $\boldsymbol{\mathrm{g}}_{y}$.
	
	
	Therefore, for each update of $u_{qr}$, it is enough to update $g_{qr}$ alone and it is not necessary to update $\boldsymbol{\mathrm{g}}_{q \ast}$ of all the columns of $q^{th}$ row and element importance $z_{qr}$. 
	
\end{proof}

\subsubsection*{Updating solution for factor matrix $\boldsymbol{\mathrm{V}}$:}

Taking the first and second order partial derivatives of the function $f$ with respect to $\boldsymbol{\mathrm{V}}$, we have a new solution for the gradient $\boldsymbol{\mathrm{G}}$ and second order derivative $\boldsymbol{\mathrm{H}}$ as,

\begin{equation}
\label{eq_pdei}
\frac{\partial f}{\partial \boldsymbol{\mathrm{V}}} = \boldsymbol{\mathrm{G}} =  -\boldsymbol{\mathrm{X_2}}(\boldsymbol{\mathrm{W}} \odot \boldsymbol{\mathrm{U}}) + \boldsymbol{\mathrm{V}}(\boldsymbol{\mathrm{U}}^T\boldsymbol{\mathrm{U}} \ast \boldsymbol{\mathrm{W}}^T\boldsymbol{\mathrm{W}}),
\end{equation}

\begin{equation}
\label{eq_ppdei}
\frac{\partial^2 f}{\partial \boldsymbol{\mathrm{V}}} = \boldsymbol{\mathrm{H}} =  \boldsymbol{\mathrm{U}}^T\boldsymbol{\mathrm{U}} \ast \boldsymbol{\mathrm{W}}^T\boldsymbol{\mathrm{W}}.
\end{equation}

With the updated $\boldsymbol{\mathrm{G}}$ and $\boldsymbol{\mathrm{H}}$ as per Equations~(\ref{eq_pdei}) and~(\ref{eq_ppdei}), the element importance is calculated according to Section 4.1 and the factor matrix update is performed in similar fashion.

\subsubsection*{Updating solution for factor matrix $\boldsymbol{\mathrm{W}}$:}

Taking the first and second order partial derivatives of the function $f$ with respect to $\boldsymbol{\mathrm{W}}$, we have a new solution for the gradient $\boldsymbol{\mathrm{G}}$ and second order derivative $\boldsymbol{\mathrm{H}}$ as,

\begin{equation}
\label{eq_pdet}
\frac{\partial f}{\partial \boldsymbol{\mathrm{W}}} = \boldsymbol{\mathrm{G}} =  -\boldsymbol{\mathrm{X_3}}(\boldsymbol{\mathrm{V}} \odot \boldsymbol{\mathrm{U}}) + \boldsymbol{\mathrm{W}}(\boldsymbol{\mathrm{U}}^T\boldsymbol{\mathrm{U}} \ast \boldsymbol{\mathrm{V}}^T\boldsymbol{\mathrm{V}}),
\end{equation}

\begin{equation}
\label{eq_ppdet}
\frac{\partial^2 f}{\partial \boldsymbol{\mathrm{W}}} = \boldsymbol{\mathrm{H}} =  \boldsymbol{\mathrm{U}}^T\boldsymbol{\mathrm{U}} \ast \boldsymbol{\mathrm{V}}^T\boldsymbol{\mathrm{V}}.
\end{equation}

With the updated $\boldsymbol{\mathrm{G}}$ and $\boldsymbol{\mathrm{H}}$ as per Equations~(\ref{eq_pdet}) and~(\ref{eq_ppdet}), the element importance is calculated according to Section 4.1 and the factor matrix update is performed in similar fashion. 

Algorithm~\ref{alg:one} details the process. 
\begin{algorithm}
	\SetAlgoLined
	\textbf{Input:} Tensor \( \boldsymbol{\mathcal{X}} \); Randomly Initialized factor matrices \(\boldsymbol{\mathrm{U}} \in \R^{Q \times R}\), \(\boldsymbol{\mathrm{V}} \in \R^{P \times R}\), \(\boldsymbol{\mathrm{W}} \in \R^{S \times R}\); Rank \( R \); \(\boldsymbol{\mathrm{Z}} = \varnothing\); Number of rows in any given factor matrix $rows$; Maximum number of iterations $maxiters$. \\
	\textbf{Output:} Learned Factor matrices \(\boldsymbol{\mathrm{U}}\), \(\boldsymbol{\mathrm{V}}\),\(\boldsymbol{\mathrm{W}}\) \\
	\For{$k = 1: maxiters$}{
		compute \( \boldsymbol{\mathrm{G}} \) and \( \boldsymbol{\mathrm{H}} \) using Equations~\eqref{eq_9} and \eqref{eq_10}; \(L = \norm{\boldsymbol{\mathrm{H}} }\); \\
		\uIf{$k$ == 1}{
			\For{\(r = 1:R\)}
			{
				\For{\(q = 1: rows\)}
				{
					compute element importance $z_{qr}$ using Equation \eqref{eq_20}; \\
					store the initial element importance, $z_{qr}^{k-1} = z_{qr}$; \\
					\If{$z_{qr} > 0$}
					{
						update the element $u_{qr}$ using Equation \eqref{eq_13};
					}
				}
				
			}
		}
		\Else{
			\For{\(r = 1:R\)}{
				\For{$q = 1: rows$}{
					compute element importance $z_{qr}$ using Equation \eqref{eq_20}; \\
					identify	the saturation point~$sp_{qr}$ using Equation \eqref{eq_24} or Equation \eqref{eq_27}; \\
					\If{$sp_{qr}>0$}{
						update the element	$u_{qr}$ using Equation~\eqref{eq_13};
					}
					update the previous element importance, $z_{qr}^{k-1} = z_{qr}$;
				}
			}
		}
		repeat analogues lines 4 to 26  with $\boldsymbol{\mathrm{G}}$ and $\boldsymbol{\mathrm{H}}$  calculated using Equations \eqref{eq_pdei} and \eqref{eq_ppdei} respectively to update elements of $\boldsymbol{\mathrm{V}}$ in lines 11 and 21; \\
		repeat analogues lines 4 to 26  with $\boldsymbol{\mathrm{G}}$ and $\boldsymbol{\mathrm{H}}$ calculated using Equations \eqref{eq_pdet} and \eqref{eq_ppdet} respectively to update elements of $\boldsymbol{\mathrm{W}}$ in lines 11 and 21; \\
	}
	\caption{Saturating Coordinate Descent (SaCD) Algorithm}
	\label{alg:one}
\end{algorithm}

\subsection{Fast Saturating Coordinate Descent (FSaCD) using Parallelization}
In this section we propose the Fast SaCD algorithm that leverages the column-wise element update to speed up the factorization process of SaCD. Instead of pre-computing the gradient using Equation~\eqref{eq_9}, we propose to calculate the gradient column-wise during the column-wise element update. To further improve the performance of SaCD, we utilize the multiple cores of the single machine. 

The pre-calculation of $\boldsymbol{\mathrm{G}}$ as per Equation \eqref{eq_9} consists of an expensive $mttkrp$ operation (i.e., $\boldsymbol{\mathrm{X_1}}(\boldsymbol{\mathrm{V}} \odot \boldsymbol{\mathrm{W}})$), a complex step in factorization that causes the Intermediate Data Explosion (IDE). IDE is caused due to the materialization and storage of the intermediate data ($\boldsymbol{\mathrm{V}} \odot \boldsymbol{\mathrm{W}}$). The calculation of $mttkrp$ is a rather infamous computational kernel with a lot of related work attempting to optimize the computations \cite{smith2015splatt,choi2014dfacto,bader2007efficient}. We utilize the concept of sparse tensor times vector product ($sttvp$) which redesigns the NTF algorithm to calculate the column-wise $mttkrp$ \cite{TTB_Software}. $sttvp$ simplifies $mttkrp$ by multiplying sparse tensor to a vector instead of multiplying it to a matrix and minimizes the IDE. We call this Fast SaCD algorithm as FSaCD. We propose to calculate the element importance and $mttkrp$ column-wise and parallelize the calculation together. 

Based on Algorithm~\ref{alg:one}, we calculate the element importance column-wise so that only the respective column of gradients is needed. It enables us to calculate the $mttkrp$ of only one column at a time as follows, 

For simplicity, let us represent $mttkrp$ as,
\begin{equation}
\boldsymbol{\mathrm{M}} = \boldsymbol{\mathrm{X_1}}(\boldsymbol{\mathrm{W}} \odot \boldsymbol{\mathrm{V}}).
\end{equation}

The $mttkrp$ operation is simplified using the sparseness of tensor and the column-wise $mttkrp$ ($\boldsymbol{\mathrm{M}}$) can be calculated as,

\begin{equation}
\boldsymbol{\mathrm{m}}_r = \sum_{(q,p,s) \in \Omega^U_q}(x_{qps}(\boldsymbol{\mathrm{w}}_r \boldsymbol{\mathrm{v}}_r)),
\end{equation}
where $\Omega^U_q$ indicates a subset of $\Omega$ whose mode $U$'s index is $q$. $\boldsymbol{\mathrm{v}}_r$ and $\boldsymbol{\mathrm{w}}_r$ indicates the $r^{th}$ column of the factor matrices $\boldsymbol{\mathrm{V}}$ and $\boldsymbol{\mathrm{W}}$ respectively.
 
Now, the column-wise gradient to solve $\boldsymbol{\mathrm{U}}$ is calculated as,
\begin{equation}
\label{eq_colgrad}
\boldsymbol{\mathrm{g}}_r = \boldsymbol{\mathrm{m}}_r + (\boldsymbol{\mathrm{U}}(\boldsymbol{\mathrm{V}}^T\boldsymbol{\mathrm{V}} \ast \boldsymbol{\mathrm{W}}^T\boldsymbol{\mathrm{W}}))_r.
\end{equation}

The column-wise element importance is calculated as,
\begin{equation}
\label{eq_coleq}
\boldsymbol{\mathrm{z}}_r = -(\boldsymbol{\mathrm{g}}_r \ast \boldsymbol{\mathrm{\hat{u}}}_r) -\frac{L}{2}(\boldsymbol{\mathrm{\hat{u}}}_r\ast\boldsymbol{\mathrm{\hat{u}}}_r).
\end{equation}

With the calculated gradient (Equation~\eqref{eq_colgrad}) and element importance (Equation~\eqref{eq_coleq}), the column-wise update is performed as,

\begin{equation}
\label{eq_upur}
\boldsymbol{\mathrm{u}}_{r} \leftarrow \boldsymbol{\mathrm{u}}_{r}+ \max(0,\frac{\boldsymbol{\mathrm{g}}_{r}}{\boldsymbol{\mathrm{h}}_{rr}}),
\end{equation}
where \(\boldsymbol{\mathrm{u}}_{r}\) indicates the \(r^{th}\) column of the factor matrix \(\boldsymbol{\mathrm{U}}\) and $\boldsymbol{\mathrm{h}}_{rr} = (\boldsymbol{\mathrm{V}}^T\boldsymbol{\mathrm{V}} \ast \boldsymbol{\mathrm{W}}^T\boldsymbol{\mathrm{W}})_{rr}$.

The above process shows the learning of factor matrix \(\boldsymbol{\mathrm{U}}\). Next we briefly show that the process is applicable to other factor matrices $\boldsymbol{\mathrm{V}}$ and $\boldsymbol{\mathrm{W}}$ too. 

\subsubsection*{Updating solution for factor matrix $\boldsymbol{\mathrm{V}}$:}

The $mttkrp$ ($\boldsymbol{\mathrm{M}} = \boldsymbol{\mathrm{X_2}}(\boldsymbol{\mathrm{W}} \odot \boldsymbol{\mathrm{U}})$) operation in Equation~\eqref{eq_pdei} can be calculated column-wise as,

\begin{equation}
\boldsymbol{\mathrm{m}}_r = \sum_{(q,p,s) \in \Omega^V_p}(x_{qps}(\boldsymbol{\mathrm{u}}_r \boldsymbol{\mathrm{w}}_r)).
\end{equation}

Now, the column-wise gradient to solve $\boldsymbol{\mathrm{V}}$ is computed as,
\begin{equation}
\label{eq_colgradI}
\boldsymbol{\mathrm{g}}_r = \boldsymbol{\mathrm{m}}_r + (\boldsymbol{\mathrm{V}}(\boldsymbol{\mathrm{U}}^T\boldsymbol{\mathrm{U}} \ast \boldsymbol{\mathrm{W}}^T\boldsymbol{\mathrm{W}}))_r.
\end{equation}

The column-wise element importance is calculated as,
\begin{equation}
\label{eq_coleqI}
\boldsymbol{\mathrm{z}}_r = -(\boldsymbol{\mathrm{g}}_r \ast \boldsymbol{\mathrm{\hat{v}}}_r) -\frac{L}{2}(\boldsymbol{\mathrm{\hat{v}}}_r\ast\boldsymbol{\mathrm{\hat{v}}}_r).
\end{equation}

With the calculated gradient (Equation~\eqref{eq_colgradI}) and element importance (Equation~\eqref{eq_coleqI}), the column-wise update is performed as,
\begin{equation}
\label{eq_upir}
\boldsymbol{\mathrm{v}}_{r} \leftarrow \boldsymbol{\mathrm{v}}_{r}+\max(0,\frac{\boldsymbol{\mathrm{g}}_{r}}{\boldsymbol{\mathrm{h}}_{rr}}),
\end{equation}
where \(\boldsymbol{\mathrm{v}}_{r}\) indicates the \(r^{th}\) column of the factor matrix \(\boldsymbol{\mathrm{V}}\) and $\boldsymbol{\mathrm{h}}_{rr} = (\boldsymbol{\mathrm{U}}^T\boldsymbol{\mathrm{U}} \ast \boldsymbol{\mathrm{W}}^T\boldsymbol{\mathrm{W}})_{rr}$.

\subsubsection*{Updating solution for factor matrix $\boldsymbol{\mathrm{W}}$:}

The $mttkrp$ ($\boldsymbol{\mathrm{M}} = \boldsymbol{\mathrm{X_3}}(\boldsymbol{\mathrm{V}} \odot \boldsymbol{\mathrm{U}})$) operation in Equation~\eqref{eq_pdet} can be calculated column-wise as,

\begin{equation}
\boldsymbol{\mathrm{m}}_r = \sum_{(q,p,s) \in \Omega^W_s}(x_{qps}(\boldsymbol{v}_r \boldsymbol{u}_r)).
\end{equation}

Now, the column-wise gradient to solve $\boldsymbol{\mathrm{W}}$ is computed as,

\begin{equation}
\label{eq_colgradT}
\boldsymbol{\mathrm{g}}_r = \boldsymbol{\mathrm{m}}_r + (\boldsymbol{\mathrm{W}}(\boldsymbol{\mathrm{U}}^T\boldsymbol{\mathrm{U}} \ast \boldsymbol{\mathrm{V}}^T\boldsymbol{\mathrm{V}}))_r.
\end{equation}

The column-wise element importance is calculated as,
\begin{equation}
\label{eq_coleqT}
\boldsymbol{\mathrm{z}}_r = -(\boldsymbol{\mathrm{g}}_r \ast \boldsymbol{\mathrm{\hat{w}}}_r) -\frac{L}{2}(\boldsymbol{\mathrm{\hat{w}}}_r\ast\boldsymbol{\mathrm{\hat{w}}}_r).
\end{equation}

With the calculated gradient (Equation~\eqref{eq_colgradT}) and element importance (Equation~\eqref{eq_coleqT}), the column-wise update is performed as,
\begin{equation}
\label{eq_uptr}
\boldsymbol{\mathrm{w}}_{r} \leftarrow \boldsymbol{\mathrm{w}}_{r}+\max(0,\frac{\boldsymbol{\mathrm{g}}_{r}}{\boldsymbol{\mathrm{h}}_{rr}}),
\end{equation}
where \(\boldsymbol{\mathrm{w}}_{r}\) indicates the \(r^{th}\) column of the factor matrix \(\boldsymbol{\mathrm{W}}\) and $\boldsymbol{\mathrm{h}}_{rr} = (\boldsymbol{\mathrm{U}}^T\boldsymbol{\mathrm{U}} \ast \boldsymbol{\mathrm{V}}^T\boldsymbol{\mathrm{V}})_{rr}$.

With the column-wise gradients calculated as per Equations~\eqref{eq_colgrad}, \eqref{eq_colgradI}, and \eqref{eq_colgradT}, and the column-wise element importance calculated as per Equations~\eqref{eq_coleq}, \eqref{eq_coleqI}, and \eqref{eq_coleqT}, the factor matrix update is performed in parallel as detailed in the Algorithm 2.

\begin{algorithm}
	\SetAlgoLined
	\textbf{Input:} Tensor \( \boldsymbol{\mathcal{X}} \); Randomly Initialized factor matrices \(\boldsymbol{\mathrm{U}} \in \R^{Q \times R}\), \(\boldsymbol{\mathrm{V}} \in \R^{P \times R}\), \(\boldsymbol{\mathrm{W}} \in \R^{S \times R}\); Rank \( R \); \(\boldsymbol{\mathrm{Z}} = \varnothing\); Number of rows in any given factor matrix $rows$; Maximum number of iterations $maxiters$. \\
	\textbf{Output:} Learned Factor matrices \(\boldsymbol{\mathrm{U}}\), \(\boldsymbol{\mathrm{V}}\),\(\boldsymbol{\mathrm{W}}\) \\
	\For{$k = 1:maxiters$}{
		compute \( \boldsymbol{\mathrm{H}} \) using Equation~\eqref{eq_10}; \(L = \norm{\boldsymbol{\mathrm{H}} }\); \\
		\uIf{$k$ == 1}{
			\textbf{Parallel} \For{\(r = 1:R\)}{
				compute gradient $\boldsymbol{\mathrm{g}}_r$ using Equation~\eqref{eq_colgrad}; \\
				compute element importance $\boldsymbol{\mathrm{z}}_r$ using Equation~\eqref{eq_coleq}; \\
				store the initial element importance, $\boldsymbol{\mathrm{z}}_{r}^{k-1} \leftarrow \boldsymbol{\mathrm{z}}_{r}$; \\
				\If{\(\boldsymbol{\mathrm{z}}_{r} > 0\)}{
					update $\boldsymbol{\mathrm{u}}_r$ using Equation~\eqref{eq_upur};
				}
				
			}
		}
		\Else{
			\textbf{Parallel} \For{\(r = 1:R\)}{
				compute gradient $\boldsymbol{\mathrm{g}}_r$ using Equation~\eqref{eq_colgrad}; \\
				compute element importance $\boldsymbol{\mathrm{z}}_r$ using Equation~\eqref{eq_coleq}; \\
				\For{$q = 1: rows$}{
					identify the saturation point~$sp_{qr}$ using Equation~\eqref{eq_24} or Equation~\eqref{eq_27}; \\
					\If{$sp_{qr}>0$}{
					update the element	$u_{qr}$ using Equation~\eqref{eq_13};
					}
					update the previous element importance, $z_{qr}^{k-1} = z_{qr}$;
				}
			}
		}
	repeat analogues lines 4 to 26 with $\boldsymbol{\mathrm{H}}$, $\boldsymbol{\mathrm{g}}_r$ and $\boldsymbol{\mathrm{z}}_r$ calculated using \eqref{eq_ppdei}, \eqref{eq_colgradI} and \eqref{eq_coleqI} respectively to update the elements of $\boldsymbol{\mathrm{V}}$ in lines 11 and 21; \\
	repeat analogues lines 4 to 26 with $\boldsymbol{\mathrm{H}}$, $\boldsymbol{\mathrm{g}}_r$ and $\boldsymbol{\mathrm{z}}_r$ calculated using \eqref{eq_ppdet}, \eqref{eq_colgradT} and \eqref{eq_coleqT} respectively to update the elements of $\boldsymbol{\mathrm{W}}$ in lines 11 and 21; \\
	}
\caption{Fast Saturating Coordinate Descent (FSaCD) Algorithm}
\label{alg:two}
\end{algorithm}

\section{Theoretical Analysis}

We analyze SaCD in terms of convergence, time complexity, and memory requirement. We use the following symbols in the analysis: $R$ (rank), $K$ (maximum number of iterations), $M$ (number of factor matrices, and a three-mode tensor $\boldsymbol{\mathcal{X}} \in \R^{(Q\times Q \times Q)}$.

\subsection{Convergence Analysis}
In this section, we analyze the convergence of SaCD under the following assumptions.

\textbf{\textit{Assumption 1}}. The objective function $f$ w.r.t each factor matrix $\nabla f(u^k)$ is continuous, differentiable and holds the Lipschitz continuity.

\textbf{\textit{Assumption 2}}. Each element importance is calculated by Equation~\eqref{eq_20} for all iterations $K$ and the parameter $L^{(k-1)}$ obeys $l\leq L^{(k-1)}\leq L$.

\begin{lemma}
\label{lemma3}
Based on the $Assumptions$ $1$ and $2$, for given $K$ iterations, $\sum_{k=1}^\infty \norm{u^k - u^{k+1}}^2 < \infty$.
\end{lemma}
\begin{proof}
For the element selection as per Equation~\eqref{eq_24}, we have inequality and therefore,

\begin{equation}
\label{eq_28}
f(u^k)-f(u^{k-1}) > L^{k}\norm{u^{k-1}-u^{k}}^2 - L^{k-1}\norm{u^{k-2}-u^{k-1}}^2. 
\end{equation}

If we sum the above inequality over $k$ from $1$ to $K$, we have
\begin{equation}
\label{eq_29}
f(u^1)-f(u^K) \geq \sum_{k=1}^K L^{k} \norm{u^{k-1}-u^{k}}^2-L^{k-1} \norm{u^{k-2}-u^{k-1}}
\end{equation}
\begin{equation*}
\geq \sum_{k=1}^K L^{k} \norm{u^{k-1}-u^{k}}^2 \geq \sum_{k=1}^K l \norm{u^{k-1}-u^{k}}^2.
\end{equation*}
As the function $f$ is lower bounded, for $k = \infty$, the proof satisfies. 
\end{proof}

The update rule in Equation~\eqref{eq_17} utilizes the Newton method \cite{lee2001algorithms} to apply the updates to the set of important elements that are identified by SaCD.
\begin{theorem}
	The newton method update rule using SaCD converges faster to the optimal solution by reaching the saturation point.
\end{theorem}
\begin{proof}
	The Newton method based update uses the update sequence,
	\begin{equation}
	\label{eq_31}
	u^{k+1} = \max(0, u^k - \frac{f'(u^k)}{f''(u^k)}), k=0,1, \dots
	\end{equation}
	where \(k\) indicates the \(k^{th}\) iteration.
	Using SaCD, we select \( e_{qr} \in \boldsymbol{\mathrm{E}} \) where $\boldsymbol{\mathrm{E}}$ is a set of elements to be updated in \(\boldsymbol{\mathrm{U}}\). The gradients and second order derivatives for the sequence can be defined as:
	
	\begin{equation}
	\label{eq_32}
	f'(u) = - \sum_{q,r \in \boldsymbol{\mathrm{E}}} g^k_{qr},
	\end{equation}
	
	\begin{equation}
	\label{eq_33}
	f''(u) = \sum_{q,r} h^k_{rr} >  \sum_{q,r \in \boldsymbol{\mathrm{E}}} h^k_{rr} > 0.
	\end{equation}
	As per linear differential equations properties, for any positive \(f''(u),f'(b) \leq f'(u)+(b-u)f''(u)  \forall u,b \geq 0 \).
	As we know that \(f''(u) > 0\), and setting \(b=u-(f'(u))/(f''(u) )\)  , we have
	\begin{equation}
	\label{eq_34}
	f'(u - \frac{f'(u)}{f''(u)}) \leq 0.
	\end{equation}
	
	With the initialized \(u\), suppose the update sequence that holds these properties will converge to a saturation point \(u^{sat} = \lim_{k\to\infty} u^k\) where \(f'(u^{sat}) \leq 0\).  For the larger value of  \(k\) by continuity of given gradients and second-order derivatives,
	\begin{equation}
	\label{eq_35}
	\frac{f'(u^k)}{f''(u^k)} < \frac{u^{sat}}{2f''(u^{sat})},
	\end{equation}
	\begin{equation}
	\label{eq_36}
	u^{sat} - u^k < \frac{u^{sat}}{2f''(u^{sat})}.
	\end{equation}
	
	From \eqref{eq_35} and \eqref{eq_36}, we have \(u^{sat}-u^{k+1}  <0 \) that literally contradicts \(u^{k+1} \leq u^{sat}\). Hence, it can be said that the proposed algorithm will converge faster to the optimal solution.
\end{proof}

\subsection{Time Complexity}
\begin{lemma}
\label{lemma4}
The time complexity of SaCD is $O(|\boldsymbol{\mathcal{X}}|+(M+K)(Q^3 R+ QR^2+QR+1))$
\end{lemma}
\begin{proof}
Algorithm~\ref{alg:one} includes five operations: initialization of factor matrices and respective importance matrix $\boldsymbol{\mathrm{V}}$, unfolding of the tensor, gradient calculation, updating of factor matrices, and updating of importance matrix. 

The random initialization of  $M$ number of factor matrices and respective importance matrix takes $O(2MQR)$. Unfolding the tensor generally takes $O(|\boldsymbol{\mathcal{X}}|)$. SaCD requires gradients $\boldsymbol{\mathrm{G}}$ to be calculated before updating the elements that involves the calculation of two terms $-\boldsymbol{\mathrm{X}}_1 (\boldsymbol{\mathrm{W}} \odot \boldsymbol{\mathrm{V}} )$ and $\boldsymbol{\mathrm{U}}(\boldsymbol{\mathrm{V}}\boldsymbol{\mathrm{V}}^T \ast \boldsymbol{\mathrm{W}}\boldsymbol{\mathrm{W}}^T)$ as shown in Equation~\eqref{eq_9} and Equation~\eqref{eq_14} and requires $O(Q^3 R),O(QR^2)$ respectively. Let $E$ be the total number of elements selected for update. Updating each factor matrix takes $O(E)$ and calculating $\boldsymbol{\mathrm{Z}}$ takes additional $O(E)$, where $E \leq QR$. For each iteration, the value of $E$ reduces and at some point, it will reach $0$. But if $\boldsymbol{\mathrm{Z}}$ can be kept in memory, it can be updated while updating each entry for $O(1)$. Thus, the time complexity of SaCD can be formulated as $O(|\boldsymbol{\mathcal{X}}|+(M+K)(Q^3 R+ QR^2+QR+1))$.
\end{proof}
The time complexity of a third-order tensor factorization algorithms, including SaCD, remains to be cubical ($Q^3 R$). However, the element selection in SaCD avoids frequent element updates and controls the complexity as the number of element updates $E \leq QR$.
\subsection{Memory Requirement}
\begin{lemma}
\label{lemma5}
The memory requirement of SaCD is $O(|\boldsymbol{\mathcal{X}}|+MQ(3R+ Q))$.
\end{lemma}
\begin{proof}
For the factorization of an input tensor $\boldsymbol{\mathcal{X}} \in \R^{Q \times Q \times Q}$, SaCD stores the following types of data in the memory at each iteration: Unfolded tensor; Factor matrices and respective importance matrices; and precomputed gradient and second order derivative. The unfolded tensor with respective to any mode requires $O(|\boldsymbol{\mathcal{X}}|)$. The $M$ number of factor matrices require $O(MQR)$ while the respective importance matrices require the same amount of memory $O(MQR)$. The precomputed gradients which is of the same size as the respective factor matrices will require another $O(MQR)$. The Hadamard matrix which is a square matrix for each factor matrix will require an additional $O(MQ^2)$ memory. Thus, the memory requirement of SaCD is $O(|\boldsymbol{\mathcal{X}}|+MQ(3R+ Q))$.
\end{proof}

\section{Empirical Analysis}
We would like to validate that SaCD can perform the factorization process efficietly as well as accurately. Experiments were conducted to answer the following questions: 

\textbf{Q1}. How scalable is SaCD? What is its runtime performance?

\textbf{Q2}. How accurately can SaCD predict missing values and can be used in recommendation?

\textbf{Q3}. How accurately SaCD identify the unique patterns and can be used in pattern mining?

\textbf{Q4}. What is the impact of parallelization on SaCD?

\subsection{Datasets}
Several real-world and synthetic datasets were used to evaluate the performance of SaCD in comparison to the state-of-the-art algorithms. Table~\ref{table_datasets} details the four real-world datasets used. Delicious\footnote{https://del.icio.us/} consists of $1797$ user’s tagging behavior on 24073 URLs with $15752$ tags. LastFM\footnote{https://www.last.fm/} consists of $1583$ users, $8383$ artists and $3886$ tags associated with the artists. Movielens\footnote{https://movielens.org/} consists of $1129$, $3884$ and $3693$ user, movies and tags respectively.  Gowalla\footnote{http://www.yongliu.org/datasets/}, the LBSN Foursquare dataset, records the $1$ $Million$ users' check-in activity at $2$ $Million$ locations. 
\begin{table}[!t]
\renewcommand{\arraystretch}{1.3}
\caption{Real-world Dataset Details. (M: Million)}
\label{table_datasets}
\centering
\begin{tabular}{lll}
\toprule
\bfseries Dataset & \bfseries Tensor Size & \bfseries Density\\
\toprule
Delicious  & $1797 \times 24073 \times 15752$   & $0.0000003$\\
LastFM & $1583 \times 8383 \times 3886$  & $0.000002$\\
Movielens & $1129 \times 3884 \times 3693$  & $0.000001$\\
Gowalla & $1M \times 2M \times 24$  &  $0.000001$\\
\bottomrule
\end{tabular}
\end{table}
\subsection{Experimental Setup and Benchmarks}
The source codes of SaCD and its parallelized version have been made available\footnote{https://github.com/thirubs/SaCD}. All single-core experiments were executed on $Intel$ $(R)$ $Core^{TM}$ $i7-6600U$ $CPU$ $@$ $2.60GHz$ model with $16GB$ $RAM$. The multi-cores experiments in section 6.4 were executed on $Intel$ $(R)$ $Xeon(R)$ $CPU$ $E5-2680$ $v3$ $@$ $2.50GHz$ model with $12GB$ $RAM$ and $12$ $cores$. For real-world datasets, we use $5$ fold cross validation with $80$\% of data used for training and $20$\% for testing. 

We compare SaCD with the following benchmark algorithms.
\begin{enumerate}
    \item \textbf{APG} \cite{zhang2016fast} uses gradients to accelerate the convergence. The objective function is smoothed using the proximal gradients and the gradient calculation is simplified using the low-rank approximations. Instead of calculating the gradient using the original tensor $\boldsymbol{\mathcal{X}}$, the low-rank approximations of the original tensor $\boldsymbol{\mathcal{X}}$ is used to calculate the gradients.  
    \item \textbf{FMU}~\cite{phan2012fast} and \textbf{FHALS}~\cite{phan2013fast} are optimized variations of MU and ALS respectively. The tensor unfolding and the Kronecker product during the gradient calculations are simplified to minimize the computational cost.
    \item \textbf{BCDP}~\cite{xu2013block} decomposes the non-convex optimization function into multiple blocks of convex problem. And the convex blocks are solved cyclically to update the factor matrices. The convex blocks are smoothed for fast convergence.
    \item \textbf{CDTF}~\cite{shin2017fully} is a latest CD algorithm for TF. For better scalability, the factor matrices are updated column-wise alternatively. For a fair comparision, we use the serial version of CDTF for NTF.
    \item \textbf{GCD}~\cite{balasubramaniam2018nonnegative} is a element selection based CD algorithm which selects important elements and update repeatedly for fast convergence. The elements are selected row-wise using the frequent gradient updates.
\end{enumerate}

\subsection{Evalution Criteria}
Tensor completion has its special properties, that discriminate it with factorization, such as the effect of the missing values on the rank/regularization selection and the optimization method. The Root Means Square Error ($RMSE$) is a commonly used metric to evaluate the tensor approximation performance. The recommendation quality is evaluated using precision, recall, and F1 score,

\begin{equation}
\label{eq_30}
RMSE = \sqrt{\frac{\sum(\boldsymbol{\mathcal{X}}_{test} - \boldsymbol{\mathcal{\hat{X}}}_{test})^2}{n}},
\end{equation}
where $\boldsymbol{\mathcal{X}}_{test}$ is the test data,  $\boldsymbol{\mathcal{\hat{X}}}_{test}$ is approximated data and $n$ is the number of elements in test data. 

\begin{equation}
    \label{eq_precision}
    \text{Precision} = \frac{True~Positive}{True~Positive + False~Positive}.
\end{equation}

\begin{equation}
    \label{eq_recall}
    \text{Recall} = \frac{True~Positive}{True~Positive + False~Negative}.
\end{equation}
\begin{equation}
\label{eq_f1}
    \text{F1 score} = 2\left(\frac{Precision \times Recall}{Precision + Recall}\right).
\end{equation}

We propose to use Pattern Distinctiveness ($PD$) to evaluate the quality of patterns learned using NTF as follows. 
    \begin{equation}
    \label{eq_pd}
    PD = Cosine(\boldsymbol{\mathrm{w}}_i,\boldsymbol{\mathrm{w}}_r), \forall i,r \in [1,R], i<r,
    \end{equation}
where $ Cosine(\boldsymbol{\mathrm{w}}_i$, $\boldsymbol{\mathrm{w}}_r)$  indicates the cosine similarity of $i^{th}$ and $r^{th}$ column of a factor matrix $\boldsymbol{\mathrm{W}}$.

$PD$ measures the similarity of each pattern with other patterns. So higher the $PD$ value, higher is the similarity between patterns. Since the objective is to identify unique patterns, lower the $PD$ value, better the quality of learned patterns is demonstrated.

\subsection{Scalability Analysis}
\label{scalabilityanalysis}
We evaluate the scalability of SaCD and other algorithms, with regards to size (mode length), density, and rank of the tensor, using synthetic data of diverse characteristics. We randomly generated tensors of size ranging from $64$ $\times$ $64$ $\times$ $64$ to $16384$ $\times$ $16384$ $\times$ $16384$, density ranging from $0.001$ (dense) to $0.0000001$ (sparse) and rank ranging from $10$ to $125$. Experiments using these synthetic data show the stability of SaCD and all the benchmarking algorithms in different data characteristics. 

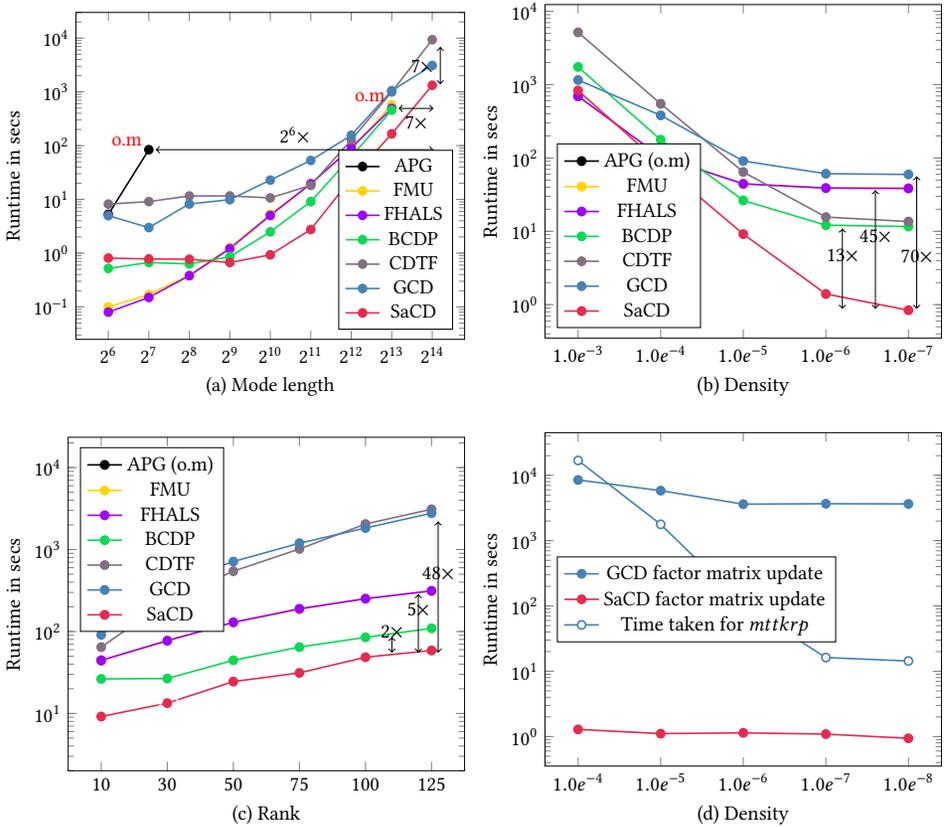
\begin{figure}
	\centering
	\subfloat{
		\pgfplotstableread[col sep=space]{dim.dat}\datatable%
		\resizebox {.45\columnwidth} {!} {
			\begin{tikzpicture}
			\begin{axis}[
			xticklabels from table={dim.dat}{xaxis},
			xlabel = {(a) Mode length},
			xtick=data,
			ymode=log,
			legend pos=south east,
			ylabel={Runtime in secs}
			]%
			
			\addplot [
			white,
			nodes near coords, 
			forget plot,
			]table[
			x expr=\coordindex,
			y=dum,     
			]{\datatable};%
			
			\addplot [
			capg,
			thick,
			mark=*,
			mark options={fill=capg},
			nodes near coords, 
			point meta=explicit symbolic, 
			]table[
			x expr=\coordindex,
			y=apg,       
			] {\datatable};%
			
			\addplot [
			cfmu,
			thick,
			mark=*,
			mark options={fill=cfmu},
			nodes near coords, 
			point meta=explicit symbolic, 
			]table[
			x expr=\coordindex,
			y=fmu,       
			] {\datatable};%
			
			\addplot [
			cfhals,
			thick,
			mark=*,
			mark options={fill=cfhals},
			nodes near coords, 
			point meta=explicit symbolic, 
			]table[
			x expr=\coordindex,
			y=fhals,       
			] {\datatable};%
			
			\addplot [
			cbcdp,
			thick,
			mark=*,
			mark options={fill=cbcdp},
			nodes near coords, 
			point meta=explicit symbolic, 
			]table[
			x expr=\coordindex,
			y=bcdp,       
			] {\datatable};%
			
			\addplot [
			ccdtf,
			thick,
			mark=*,
			mark options={fill=ccdtf},
			nodes near coords, 
			point meta=explicit symbolic, 
			]table[
			x expr=\coordindex,
			y=cdtf,       
			] {\datatable};%
			
			\addplot [
			cgcd,
			thick,
			mark=*,
			mark options={fill=cgcd},
			nodes near coords, 
			point meta=explicit symbolic, 
			]table[
			x expr=\coordindex,
			y=gcd,       
			] {\datatable};%
			
			\addplot [
			csacd,
			thick,
			mark=*,
			mark options={fill=csacd},
			nodes near coords, 
			point meta=explicit symbolic, 
			]table[
			x expr=\coordindex,
			y=sacd,       
			] {\datatable};%
			
			\addplot [
			red,
			nodes near coords, 
			forget plot,
			only marks,
			point meta=explicit symbolic, 
			every node near coord/.style={anchor=-20} 
			]table[
			x expr=\coordindex,
			y=out,     
			meta index = 17,
			] {\datatable};%
			
			\legend{APG,FMU,FHALS,BCDP,CDTF,GCD,SaCD}%
			\node (source) at (axis cs:1,84.02){};%
			\node (destination) at (axis cs:8.2,84.02){};%
			\draw[black, <->](source) -- node [above, midway] {$2^6\times$} (destination);%
			\node (source) at (axis cs:7,492.6){};%
			\node (destination) at (axis cs:8.2,492.6){};%
			\draw[black, <->](source) -- node [below, midway] {$7\times$} (destination);%
			\node (source) at (axis cs:8.2,9300){};%
			\node (destination) at (axis cs:8.2,1025){};%
			\draw[black, <->](source) -- node [left, midway] {$7\times$} (destination);%
			\end{axis}%
			\end{tikzpicture}%
	}}%
	\subfloat{
		\pgfplotstableread[col sep=space]{den.dat}\datatable
		\resizebox {.45\columnwidth} {!} {
			\begin{tikzpicture}
			\begin{axis}[
			xticklabels from table={den.dat}{xaxis},
			xtick=data,
			xlabel = {(b) Density},
			ymode=log,
			legend pos=south west,
			ylabel={Runtime in secs}
			]
			\addplot [
			white,
			nodes near coords, 
			forget plot,
			]
			table[
			x expr=\coordindex,
			y=dum,     
			] {\datatable};
			\addplot [
			capg,
			thick,
			mark=*,
			mark options={fill=capg},
			nodes near coords, 
			point meta=explicit symbolic, 
			]
			table[
			x expr=\coordindex,
			y=apg,       
			] {\datatable};
			\addplot [
			cfmu,
			thick,
			mark=*,
			mark options={fill=cfmu},
			nodes near coords, 
			point meta=explicit symbolic, 
			]
			table[
			x expr=\coordindex,
			y=fmu,       
			] {\datatable};
			\addplot [
			cfhals,
			thick,
			mark=*,
			mark options={fill=cfhals},
			nodes near coords, 
			point meta=explicit symbolic, 
			]
			table[
			x expr=\coordindex,
			y=fhals,       
			] {\datatable};
			\addplot [
			cbcdp,
			thick,
			mark=*,
			mark options={fill=cbcdp},
			nodes near coords, 
			point meta=explicit symbolic, 
			]
			table[
			x expr=\coordindex,
			y=bcdp,       
			] {\datatable};
			\addplot [
			ccdtf,
			thick,
			mark=*,
			mark options={fill=ccdtf},
			nodes near coords, 
			point meta=explicit symbolic, 
			]
			table[
			x expr=\coordindex,
			y=cdtf,       
			] {\datatable};
			\addplot [
			cgcd,
			thick,
			mark=*,
			mark options={fill=cgcd},
			nodes near coords, 
			point meta=explicit symbolic, 
			]
			table[
			x expr=\coordindex,
			y=gcd,       
			] {\datatable};
			\addplot [
			csacd,
			thick,
			mark=*,
			mark options={fill=csacd},
			nodes near coords, 
			point meta=explicit symbolic, 
			]
			table[
			x expr=\coordindex,
			y=sacd,       
			] {\datatable};
			\node (source) at (axis cs:4.1,70.69){};
			\node (destination) at (axis cs:4.1,0.7){};
			\draw[black, <->](source) -- node [below, midway] {$70\times$} (destination);
			\node (source) at (axis cs:3.6,45.69){};
			\node (destination) at (axis cs:3.6,0.7){};
			\draw[black, <->](source) -- node [above, midway] {$45\times$} (destination);
			\node (source) at (axis cs:3.2,13.69){};
			\node (destination) at (axis cs:3.2,0.7){};
			\draw[black, <->](source) -- node [above, midway] {$13\times$} (destination);
			\legend{APG (o.m),FMU,FHALS,BCDP,CDTF,GCD,SaCD}
			\end{axis}
			\end{tikzpicture}}}
	\qquad
	\subfloat{
		\pgfplotstableread[col sep=space]{ranki.dat}\datatable
		\resizebox {.45\columnwidth} {!} {
			\begin{tikzpicture}
			\begin{axis}[
			xticklabels from table={ranki.dat}{xaxis},
			xtick=data,
			xlabel = {(c) Rank},
			ymode=log,
			ymin = 2,
			legend pos=north west,
			ylabel={Runtime in secs}
			]
			\addplot [
			white,
			nodes near coords, 
			forget plot,
			]
			table[
			x expr=\coordindex,
			y=dum,     
			] {\datatable};
			\addplot [
			capg,
			thick,
			mark=*,
			mark options={fill=capg},
			nodes near coords, 
			point meta=explicit symbolic, 
			]
			table[
			x expr=\coordindex,
			y=apg,       
			] {\datatable};
			\addplot [
			cfmu,
			thick,
			mark=*,
			mark options={fill=cfmu},
			nodes near coords, 
			point meta=explicit symbolic, 
			]
			table[
			x expr=\coordindex,
			y=fmu,       
			] {\datatable};
			\addplot [
			cfhals,
			thick,
			mark=*,
			mark options={fill=cfhals},
			nodes near coords, 
			point meta=explicit symbolic, 
			]
			table[
			x expr=\coordindex,
			y=fhals,       
			] {\datatable};
			\addplot [
			cbcdp,
			thick,
			mark=*,
			mark options={fill=cbcdp},
			nodes near coords, 
			point meta=explicit symbolic, 
			]
			table[
			x expr=\coordindex,
			y=bcdp,       
			] {\datatable};
			\addplot [
			ccdtf,
			thick,
			mark=*,
			mark options={fill=ccdtf},
			nodes near coords, 
			point meta=explicit symbolic, 
			]
			table[
			x expr=\coordindex,
			y=cdtf,       
			] {\datatable};
			\addplot [
			cgcd,
			thick,
			mark=*,
			mark options={fill=cgcd},
			nodes near coords, 
			point meta=explicit symbolic, 
			]
			table[
			x expr=\coordindex,
			y=gcd,       
			] {\datatable};
			\addplot [
			csacd,
			thick,
			mark=*,
			mark options={fill=csacd},
			nodes near coords, 
			point meta=explicit symbolic, 
			]
			table[
			x expr=\coordindex,
			y=sacd,       
			] {\datatable};
			\node (source) at (axis cs:5.1,2700){};
			\node (destination) at (axis cs:5.1,45){};
			\draw[black, <->](source) -- node [above, midway] {$48\times$} (destination);
			\node (source) at (axis cs:4.8,350){};
			\node (destination) at (axis cs:4.8,45){};
			\draw[black, <->](source) -- node [above, midway] {$5\times$} (destination);
			\node (source) at (axis cs:4.4,105){};
			\node (destination) at (axis cs:4.4,45){};
			\draw[black, <->](source) -- node [above, midway] {$2\times$} (destination);
			\legend{APG (o.m),FMU,FHALS,BCDP,CDTF,GCD,SaCD};
			\end{axis}
			\end{tikzpicture}}}%
	\subfloat{
		\pgfplotstableread[col sep=space]{denmttkrp.dat}\datatable
		\resizebox {.45\columnwidth} {!} {
			\begin{tikzpicture}
			\begin{axis}[
			xticklabels from table={denmttkrp.dat}{xaxis},
			xtick=data,
			xlabel = {(d) Density},
			ymode=log,
			legend style={at={(0.03,0.5)},anchor=west},
			ylabel={Runtime in secs}
			]
			\addplot [
			cgcd,
			thick,
			mark=*,
			mark options={fill=cgcd},
			nodes near coords, 
			point meta=explicit symbolic, 
			]
			table[
			x expr=\coordindex,
			y=gcd,       
			] {\datatable};
			\addplot [
			csacd,
			thick,
			mark=*,
			mark options={fill=csacd},
			nodes near coords, 
			point meta=explicit symbolic, 
			]
			table[
			x expr=\coordindex,
			y=sacd,       
			] {\datatable};
			\addplot [
			cgcd,
			thick,
			mark=*,
			mark options={fill=white},
			nodes near coords, 
			point meta=explicit symbolic, 
			]
			table[
			x expr=\coordindex,
			y=mttkrp,       
			] {\datatable};
			\legend{GCD factor matrix update,SaCD factor matrix update,Time taken for $mttkrp$};
			\end{axis}
			\end{tikzpicture}}}%
	\caption{Scalability Analysis using synthetic datasets. FMU and FHALS shows similar runtime performance and the lines are overlapped. o.m. out of memory.}%
	\label{fig:2}%
\end{figure}

\textbf{Mode length}. We increase the mode length $Q=P=S$ of each mode from $2^6$  to $2^{14}$ with setting the tensor density and rank to $0.00001$ and $10$ respectively. We set all the algorithms to run for maximum iterations of $30$. As shown in Figure~\ref{fig:2}(a), SaCD successfully handles the tensor of size $Q=P=S=2^{14}$. Whereas APG \cite{zhang2016fast} ran out of memory for the tensor size $Q=P=S > 2^7$, and FMU \cite{phan2012fast}, FHALS \cite{phan2013fast}, and BCDP \cite{xu2013block} ran out of memory for the tensor size $Q=P=S > 2^{13}$. Overall SaCD can factorize $2^6$ to $7$ $times$ larger tensors when compared to existing algorithms. The runtime performance of SaCD is almost constant for the smaller size tensors, however, it increases linearly for large size tensors.  This is due to the matricized tensor times Khatri-Rao product ($mttkrp$), $\boldsymbol{\mathrm{X}}_1(\boldsymbol{\mathrm{W}} \odot \boldsymbol{\mathrm{V}})$ needed in Equation~\eqref{eq_9} for the gradient calculation. In general, SaCD yields a significant time saving in comparison to other algorithms, especially GCD and CDTF, due to the avoidance of frequent gradient updates.

\textbf{Density}. In this set of experiments, we fix the tensor mode length to $Q=P=S =3000$ and rank to $10$ while decreasing the density from $1.0e^{-3}$  to $1.0e^{-7}$. As shown in  Figure~\ref{fig:2}(b), SaCD is $13$ to $70$ $times$ faster than existing algorithms for the very sparse dataset. The runtime performance of SaCD improves when the sparsity increases, due to a gradual reduction in the $mttkrp$ operation, as shown in  Figure~\ref{fig:2}(d). On the other hand, the runtime performance of GCD and CDTF degardes with an increase in sparsity. In comparison to GCD (Figure~\ref{fig:2}(d)), there is a significant time saving due to avoidance of the frequent gradient update.

\textbf{Rank}. Here, we fix the tensor mode length $Q=P=S=3000$  and density to $0.00001$ while increasing the rank from $10$ to $125$.  Figure~\ref{fig:2}(c) reveals SaCD outperforms all other algorithms easily. As proven in lemma~\ref{lemma2}, SaCD avoids frequent gradients, hence the increase in rank doesn’t adversely affect its performance.

\subsection{Tensor Approximation Performance}

In addition to scalable factorization process, it is essential that the approximated (i.e., reconstructed) tensor has good accuracy. The factor matrices learned using factorization is used to reconstruct the approximated tensor that will identify missing values. Figure~\ref{fig:rmsesyn} reports the RMSE performance of all the algorithms for the synthetic datasets used in previous section to evaluate the scalability. It is evident that SaCD doesn’t compromise with accuracy for better runtime performance and produces the best result with less error in comparison to benchmarks. 

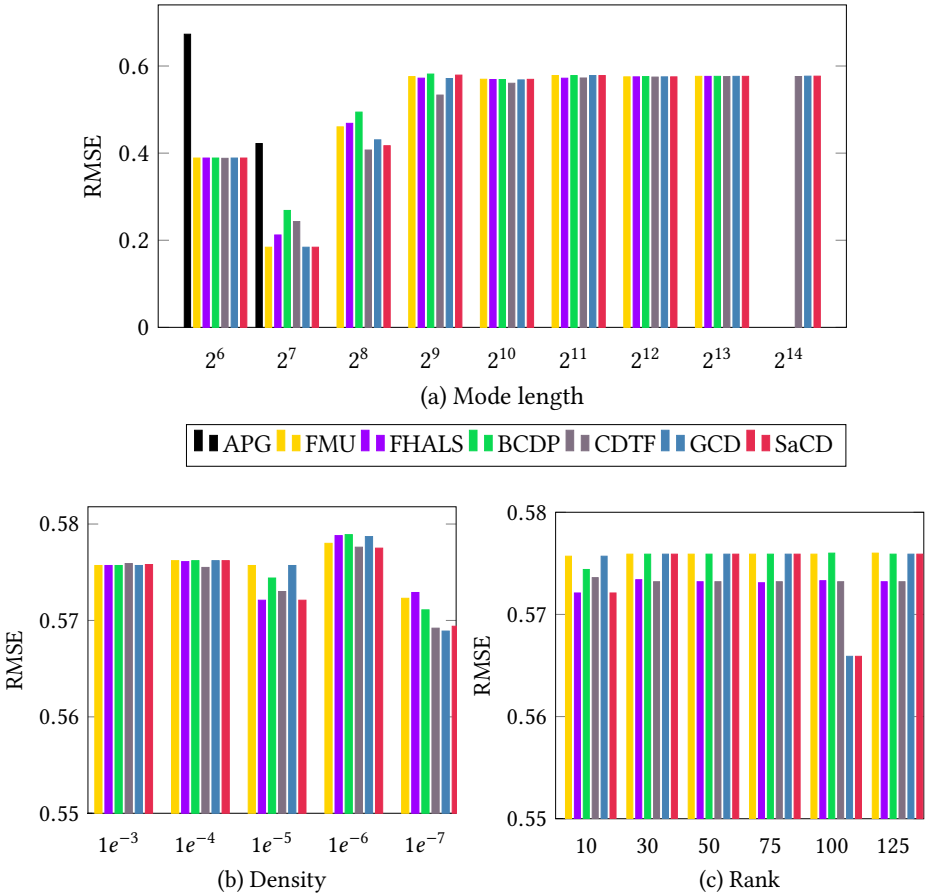
\begin{figure}[ht!]
	\centering
	\subfloat{
		\resizebox {.75\columnwidth} {!} {
			\begin{tikzpicture}
			\begin{axis}[
			width  = 0.8*\textwidth,
			height = 6cm,
			major x tick style = transparent,
			ybar=3*\pgflinewidth,
			bar width=2.5pt,
			ymajorgrids = false,
			ylabel = {RMSE},
			xlabel = {(a) Mode length},
			xticklabels={$2^6$,$2^7$,$2^8$,$2^9$,$2^{10}$,$2^{11}$,$2^{12}$,$2^{13}$,$2^{14}$},
			xtick={0,...,8},
			scaled y ticks = false,
			ymin=0,
			legend cell align=left,
			legend style={at={(1,-0.3)},anchor=north east,legend columns=-1},
			]
			\addplot[style={capg,fill=capg,mark=none}]
			coordinates {(0, 0.673) (1,0.422)};
			
			\addplot[style={cfmu,fill=cfmu,mark=none}]
			coordinates {(0, 0.3887) (1,0.184)(2,0.4603)(3,0.5757)(4,0.5696)(5,0.5781)(6,0.5751)(7,0.5763)};
			
			\addplot[style={cfhals,fill=cfhals,mark=none}]
			coordinates {(0, 0.3887) (1,0.2121)(2,0.4684)(3,0.5721)(4,0.569)(5,0.572)(6,0.5751)(7,0.5762)};
			
			\addplot[style={cbcdp,fill=cbcdp,mark=none}]
			coordinates {(0, 0.3887) (1,0.2684)(2,0.4942)(3,0.5815)(4,0.5689)(5,0.5781)(6,0.5757)(7,0.5763)};
			
			\addplot[style={ccdtf,fill=ccdtf,mark=none}]
			coordinates {(0, 0.3881) (1,0.2429)(2,0.4071)(3,0.5333)(4,0.5604)(5,0.5725)(6,0.5746)(7,0.5757)(8,0.5759)};
			\addplot[style={cgcd,fill=cgcd,mark=none}]
			coordinates {(0, 0.3887) (1,0.184)(2,0.4305)(3,0.5713)(4,0.5681)(5,0.5781)(6,0.5751)(7,0.5763)(8,0.5767)};
			
			\addplot[style={csacd,fill=csacd,mark=none}]
			coordinates {(0, 0.3887) (1,0.184)(2,0.417)(3,0.5791)(4,0.5694)(5,0.5781)(6,0.5751)(7,0.5763)(8,0.5767)};

			\legend{APG, FMU, FHALS,BCDP,CDTF, GCD, SaCD}
			\end{axis}
			\end{tikzpicture}
	}}
	\hspace{-0.6em}
	\subfloat{
		\resizebox {.45\columnwidth} {!} {
			\begin{tikzpicture}
			\begin{axis}[
			width  = 0.5*\textwidth,
			height = 6cm,
			major x tick style = transparent,
			ybar=3*\pgflinewidth,
			bar width=3pt,
			ymajorgrids = false,
			ylabel = {RMSE},
			xlabel = {(b) Density},
			xticklabels={$1e^{-3}$, $1e^{-4}$, $1e^{-5}$, $1e^{-6}$, $1e^{-7}$},
			xtick={0,...,4},
			scaled y ticks = false,
			ymin=0.55,
			legend cell align=left,
			legend style={at={(1,-0.5)},anchor=north east,legend columns=-1},
			]
			
			\addplot[style={capg,fill=capg,mark=none}]
			coordinates {(0, 0)};
			
			\addplot[style={cfmu,fill=cfmu,mark=none}]
			coordinates {(0, 0.5757) (1,0.5762)(2,0.5757)(3,0.5780)(4,0.5723)};
			
			\addplot[style={cfhals,fill=cfhals,mark=none}]
			coordinates {(0, 0.5757) (1,0.5761)(2,0.5721)(3,0.5788)(4,0.5729)};
			
			\addplot[style={cbcdp,fill=cbcdp,mark=none}]
			coordinates {(0, 0.5757) (1,0.5762)(2,0.5744)(3,0.5789)(4,0.5711)};
			
	    	\addplot[style={ccdtf,fill=ccdtf,mark=none}]
			coordinates {(0, 0.5759) (1,0.5755)(2,0.5730)(3,0.5776)(4,0.5692)};
			\addplot[style={cgcd,fill=cgcd,mark=none}]
			coordinates {(0, 0.5757) (1,0.5762)(2,0.5757)(3,0.5787)(4,0.5689)};
			
			\addplot[style={csacd,fill=csacd,mark=none}]
			coordinates {(0, 0.5758) (1,0.5762)(2,0.5721)(3,0.5775)(4,0.5694)};

			\end{axis}
			\end{tikzpicture}
	}} 
	\hspace{-0.6em}
	\subfloat{
		\resizebox {.45\columnwidth} {!} {
			\begin{tikzpicture}
			\begin{axis}[
			width  = 0.5*\textwidth,
			height = 6cm,
			major x tick style = transparent,
			ybar=3*\pgflinewidth,
			bar width=2.5pt,
			ymajorgrids = false,
			ylabel = {RMSE},
			xlabel = {(c) Rank},
			xticklabels={10, 30, 50, 75, 100, 125},
			xtick={0,...,5},
			scaled y ticks = false,
			ymin=0.55,
			ymax = 0.58,
			legend cell align=left,
			legend style={at={(1,-0.5)},anchor=north east,legend columns=-1},
			]
			
			\addplot[style={capg,fill=capg,mark=none}]
			coordinates {(0, 0)};
			
			\addplot[style={cfmu,fill=cfmu,mark=none}]
			coordinates {(0, 0.5757) (1,0.5759)(2,0.5759)(3,0.5759)(4,0.5759)(5,0.5760)};
			
			\addplot[style={cfhals,fill=cfhals,mark=none}]
			coordinates {(0, 0.5721) (1,0.5734)(2,0.5732)(3,0.5731)(4,0.5733)(5,0.5732)};
			
			\addplot[style={cbcdp,fill=cbcdp,mark=none}]
			coordinates {(0, 0.5744) (1,0.5759)(2,0.5759)(3,0.5759)(4,0.5760)(5,0.5759)};
			\addplot[style={ccdtf,fill=ccdtf,mark=none}]
			coordinates {(0, 0.5736) (1,0.5732)(2,0.5732)(3,0.5732)(4,0.5732)(5,0.5732)};
			\addplot[style={cgcd,fill=cgcd,mark=none}]
			coordinates {(0, 0.5757) (1,0.5759)(2,0.5759)(3,0.5759)(4,0.5659)(5,0.5759)};
			
			\addplot[style={csacd,fill=csacd,mark=none}]
			coordinates {(0, 0.5721) (1,0.5759)(2,0.5759)(3,0.5759)(4,0.5659)(5,0.5759)};
			
			\end{axis}
			\end{tikzpicture}
	}}
	\caption{RMSE performance of all the algorithms on synthetic datasets used to evaluate the scalability. APG does not scale for mode length $>~2^7$ whereas FMU, FHALS, and BCDP do not scale for mode length $>2^{13}$.}%
	\label{fig:rmsesyn}%
\end{figure}
\subsection{Recommendation or Prediction Performance}
The NTF problem can be considered as a solution to a recommendation or prediction problem where the estimated missing data are treated as prediction. The approximated tensor reconstructed using the factor matrices can be used to infer new values based on associations. In LastFM, Delicious, and Movielens datasets, the goal is to predict the missing entries of the tensor as accurately as possible. These entries are then inferred as “most likely items” that can be recommended to users. 

It is evident from  Figure~\ref{fig:3} that SaCD doesn’t compromise with accuracy for better runtime performance and produces the best result with less error in comparison to benchmarks. Especially in Movielens dataset, SaCD shows at least $2.5$\% accuracy improvement over other algorithms. Similar to synthetic datasets, APG and  BCDP ran out of memory ($o.m$) to process the Delicious dataset, and both FMU and FHALS ran out of time ($o.o.t$) with an increase in the rank. While all other algorithms fail to process the Gowalla dataset for higher rank, only SaCD can successfully complete the process.  The materialization  of matrices in these algorithms requires large memory making them inefficient to deal with higher ranks. It can be noticed that the higher the rank, higher is the accuracy of SaCD. However, due to the increased memory requirements, APG, BCDP, FMU, and FHALS are not suitable for higher ranks. GCD and CDTF are able to process at higher ranks, however, they are significantly slower and run out of time.

Similar performance is obtained for the measures of precision, recall, and F1 score. It can be seen in Figure~\ref{fig:rec} that SaCD significantly outperforms other algorithms in the recommendation performance on all four real-world datasets.

\begin{figure}
	\centering
	\subfloat{
		\pgfplotstableread[col sep=space]{delrmr.dat}\datatable%
		\resizebox {.35\columnwidth} {!} {
			\begin{tikzpicture}
			\begin{axis}[
			title = {Delicious},
			enlarge y limits={upper,value=0.00001},
			y tick label style = {/pgf/number format/.cd,precision = 3,/tikz/.cd},
			xticklabels from table={delrmr.dat}{xaxis},
			xlabel = {(a) Rank},
			xtick=data,
			legend pos=south west,
			ylabel={RMSE}
			]%
			
			\addplot [
			white,
			nodes near coords, 
			forget plot,
			]table[
			x expr=\coordindex,
			y=dum,     
			]{\datatable};%

			\addplot [
			cfmu,
			thick,
			mark=*,
			mark options={fill=cfmu},
			nodes near coords, 
			point meta=explicit symbolic, 
			]table[
			x expr=\coordindex,
			y=fmu,       
			] {\datatable};%
			
			\addplot [
			cfhals,
			thick,
			mark=*,
			mark options={fill=cfhals},
			nodes near coords, 
			point meta=explicit symbolic, 
			]table[
			x expr=\coordindex,
			y=fhals,       
			] {\datatable};%
			
			\addplot [
			cbcdp,
			thick,
			mark=*,
			mark options={fill=cbcdp},
			nodes near coords, 
			point meta=explicit symbolic, 
			]table[
			x expr=\coordindex,
			y=bcdp,       
			] {\datatable};%
			
			\addplot [
			ccdtf,
			thick,
			mark=*,
			mark options={fill=ccdtf},
			nodes near coords, 
			point meta=explicit symbolic, 
			]table[
			x expr=\coordindex,
			y=cdtf,       
			] {\datatable};%
			
			\addplot [
			cgcd,
			thick,
			mark=*,
			mark options={fill=cgcd},
			nodes near coords, 
			point meta=explicit symbolic, 
			]table[
			x expr=\coordindex,
			y=gcd,       
			] {\datatable};%
			
			\addplot [
			csacd,
			thick,
			mark=*,
			mark options={fill=csacd},
			nodes near coords, 
			point meta=explicit symbolic, 
			]table[
			x expr=\coordindex,
			y=sacd,       
			] {\datatable};%
			
			\addplot [
			red,
			nodes near coords, 
			forget plot,
			only marks,
			point meta=explicit symbolic, 
			every node near coord/.style={anchor=-180} 
			]table[
			x expr=\coordindex,
			y=out,     
			meta index = 9,
			] {\datatable};%
			
			\legend{FMU,FHALS,BCDP (o.m),CDTF,GCD,SaCD}%
			\end{axis}%
			\end{tikzpicture}%
	}}%
	\subfloat{
		\pgfplotstableread[col sep=space]{delrr.dat}\datatable%
		\resizebox {.35\columnwidth} {!} {
			\begin{tikzpicture}
			\begin{axis}[
			title = {Delicious},
			xticklabels from table={delrr.dat}{xaxis},
			xlabel = {(b) Rank},
			xtick=data,
			legend pos=north west,
			ylabel={Runtime in secs},
			ymode = log,
			]%
			
			\addplot [
			white,
			nodes near coords, 
			forget plot,
			]table[
			x expr=\coordindex,
			y=dum,     
			]{\datatable};%

			\addplot [
			cfmu,
			thick,
			mark=*,
			mark options={fill=cfmu},
			nodes near coords, 
			point meta=explicit symbolic, 
			]table[
			x expr=\coordindex,
			y=fmu,       
			] {\datatable};%
			
			\addplot [
			cfhals,
			thick,
			mark=*,
			mark options={fill=cfhals},
			nodes near coords, 
			point meta=explicit symbolic, 
			]table[
			x expr=\coordindex,
			y=fhals,       
			] {\datatable};%
			
			\addplot [
			cbcdp,
			thick,
			mark=*,
			mark options={fill=cbcdp},
			nodes near coords, 
			point meta=explicit symbolic, 
			]table[
			x expr=\coordindex,
			y=bcdp,       
			] {\datatable};%
			
			\addplot [
			ccdtf,
			thick,
			mark=*,
			mark options={fill=ccdtf},
			nodes near coords, 
			point meta=explicit symbolic, 
			]table[
			x expr=\coordindex,
			y=cdtf,       
			] {\datatable};%
			
			\addplot [
			cgcd,
			thick,
			mark=*,
			mark options={fill=cgcd},
			nodes near coords, 
			point meta=explicit symbolic, 
			]table[
			x expr=\coordindex,
			y=gcd,       
			] {\datatable};%
			
			\addplot [
			csacd,
			thick,
			mark=*,
			mark options={fill=csacd},
			nodes near coords, 
			point meta=explicit symbolic, 
			]table[
			x expr=\coordindex,
			y=sacd,       
			] {\datatable};%
			
			\addplot [
			red,
			nodes near coords, 
			forget plot,
			only marks,
			point meta=explicit symbolic, 
			every node near coord/.style={anchor=-180} 
			]table[
			x expr=\coordindex,
			y=out,     
			meta index = 9,
			] {\datatable};%
			
			\legend{FMU,FHALS,BCDP (o.m),CDTF,GCD,SaCD}%
			\node (source) at (axis cs:4,7938){};%
			\node (destination) at (axis cs:4,84.02){};%
			\draw[black, <->](source) -- node [above, midway] {$87\times$} (destination);%
			\end{axis}%
			\end{tikzpicture}%
	}}%
	\qquad
	\subfloat{
		\pgfplotstableread[col sep=space]{lastrmr.dat}\datatable%
		\resizebox {.35\columnwidth} {!} {
			\begin{tikzpicture}
			\begin{axis}[
			title = {LastFM},
			enlarge y limits={upper,value=0.001},
			y tick label style = {/pgf/number format/.cd,precision = 2,/tikz/.cd},
			xticklabels from table={lastrmr.dat}{xaxis},
			xlabel = {(c) Rank},
			xtick=data,
			legend pos=south west,
			ylabel={RMSE}
			]%
			
			\addplot [
			white,
			nodes near coords, 
			forget plot,
			]table[
			x expr=\coordindex,
			y=dum,     
			]{\datatable};%

			\addplot [
			cfmu,
			thick,
			mark=*,
			mark options={fill=cfmu},
			nodes near coords, 
			point meta=explicit symbolic, 
			]table[
			x expr=\coordindex,
			y=fmu,       
			] {\datatable};%
			
			\addplot [
			cfhals,
			thick,
			mark=*,
			mark options={fill=cfhals},
			nodes near coords, 
			point meta=explicit symbolic, 
			]table[
			x expr=\coordindex,
			y=fhals,       
			] {\datatable};%
			
			\addplot [
			cbcdp,
			thick,
			mark=*,
			mark options={fill=cbcdp},
			nodes near coords, 
			point meta=explicit symbolic, 
			]table[
			x expr=\coordindex,
			y=bcdp,       
			] {\datatable};%
			
			\addplot [
			ccdtf,
			thick,
			mark=*,
			mark options={fill=ccdtf},
			nodes near coords, 
			point meta=explicit symbolic, 
			]table[
			x expr=\coordindex,
			y=cdtf,       
			] {\datatable};%
			
			\addplot [
			cgcd,
			thick,
			mark=*,
			mark options={fill=cgcd},
			nodes near coords, 
			point meta=explicit symbolic, 
			]table[
			x expr=\coordindex,
			y=gcd,       
			] {\datatable};%
			
			\addplot [
			csacd,
			thick,
			mark=*,
			mark options={fill=csacd},
			nodes near coords, 
			point meta=explicit symbolic, 
			]table[
			x expr=\coordindex,
			y=sacd,       
			] {\datatable};%
			
			\addplot [
			red,
			nodes near coords, 
			forget plot,
			only marks,
			point meta=explicit symbolic, 
			every node near coord/.style={anchor=-180} 
			]table[
			x expr=\coordindex,
			y=out,     
			meta index = 9,
			] {\datatable};%
			
			\legend{FMU,FHALS,BCDP,CDTF,GCD,SaCD}%
			\end{axis}%
			\end{tikzpicture}%
	}}%
	\subfloat{
		\pgfplotstableread[col sep=space]{lastrr.dat}\datatable%
		\resizebox {.35\columnwidth} {!} {
			\begin{tikzpicture}
			\begin{axis}[
			title = {LastFM},
			xticklabels from table={lastrr.dat}{xaxis},
			xlabel = {(d) Rank},
			xtick=data,
			legend pos=north west,
			ylabel={Runtime in secs},
			ymode = log,
			]%
			
			\addplot [
			white,
			nodes near coords, 
			forget plot,
			]table[
			x expr=\coordindex,
			y=dum,     
			]{\datatable};%

			\addplot [
			cfmu,
			thick,
			mark=*,
			mark options={fill=cfmu},
			nodes near coords, 
			point meta=explicit symbolic, 
			]table[
			x expr=\coordindex,
			y=fmu,       
			] {\datatable};%
			
			\addplot [
			cfhals,
			thick,
			mark=*,
			mark options={fill=cfhals},
			nodes near coords, 
			point meta=explicit symbolic, 
			]table[
			x expr=\coordindex,
			y=fhals,       
			] {\datatable};%
			
			\addplot [
			cbcdp,
			thick,
			mark=*,
			mark options={fill=cbcdp},
			nodes near coords, 
			point meta=explicit symbolic, 
			]table[
			x expr=\coordindex,
			y=bcdp,       
			] {\datatable};%
			\addplot [
			ccdtf,
			thick,
			mark=*,
			mark options={fill=ccdtf},
			nodes near coords, 
			point meta=explicit symbolic, 
			]table[
			x expr=\coordindex,
			y=cdtf,       
			] {\datatable};%
			
			\addplot [
			cgcd,
			thick,
			mark=*,
			mark options={fill=cgcd},
			nodes near coords, 
			point meta=explicit symbolic, 
			]table[
			x expr=\coordindex,
			y=gcd,       
			] {\datatable};%
			
			\addplot [
			csacd,
			thick,
			mark=*,
			mark options={fill=csacd},
			nodes near coords, 
			point meta=explicit symbolic, 
			]table[
			x expr=\coordindex,
			y=sacd,       
			] {\datatable};%
			
			\addplot [
			red,
			nodes near coords, 
			forget plot,
			only marks,
			point meta=explicit symbolic, 
			every node near coord/.style={anchor=-180} 
			]table[
			x expr=\coordindex,
			y=out,     
			meta index = 9,
			] {\datatable};%
			
			\legend{FMU,FHALS,BCDP,CDTF,GCD,SaCD}%
			\node (source) at (axis cs:4,2200){};%
			\node (destination) at (axis cs:4,50){};%
			\draw[black, <->](source) -- node [above, midway] {$43\times$} (destination);%
			\node (source) at (axis cs:3.6,220){};%
			\node (destination) at (axis cs:3.6,40){};%
			\draw[black, <->](source) -- node [above, midway] {$5\times$} (destination);%
			\end{axis}%
			\end{tikzpicture}%
	}}%
	\qquad
	\subfloat{
		\pgfplotstableread[col sep=space]{moviermr.dat}\datatable%
		\resizebox {.35\columnwidth} {!} {
			\begin{tikzpicture}
			\begin{axis}[
			title = {Movielens},
			enlarge y limits={upper,value=0.0005},
			y tick label style = {/pgf/number format/.cd,precision = 3,/tikz/.cd},
			xticklabels from table={moviermr.dat}{xaxis},
			xlabel = {(e) Rank},
			xtick=data,
			legend pos=south west,
			ylabel={RMSE}
			]%
			
			\addplot [
			white,
			nodes near coords, 
			forget plot,
			]table[
			x expr=\coordindex,
			y=dum,     
			]{\datatable};%

			\addplot [
			cfmu,
			thick,
			mark=*,
			mark options={fill=cfmu},
			nodes near coords, 
			point meta=explicit symbolic, 
			]table[
			x expr=\coordindex,
			y=fmu,       
			] {\datatable};%
			
			\addplot [
			cfhals,
			thick,
			mark=*,
			mark options={fill=cfhals},
			nodes near coords, 
			point meta=explicit symbolic, 
			]table[
			x expr=\coordindex,
			y=fhals,       
			] {\datatable};%
			
			\addplot [
			cbcdp,
			thick,
			mark=*,
			mark options={fill=cbcdp},
			nodes near coords, 
			point meta=explicit symbolic, 
			]table[
			x expr=\coordindex,
			y=bcdp,       
			] {\datatable};%
			
			\addplot [
			ccdtf,
			thick,
			mark=*,
			mark options={fill=ccdtf},
			nodes near coords, 
			point meta=explicit symbolic, 
			]table[
			x expr=\coordindex,
			y=cdtf,       
			] {\datatable};%
			
			\addplot [
			cgcd,
			thick,
			mark=*,
			mark options={fill=cgcd},
			nodes near coords, 
			point meta=explicit symbolic, 
			]table[
			x expr=\coordindex,
			y=gcd,       
			] {\datatable};%
			
			\addplot [
			csacd,
			thick,
			mark=*,
			mark options={fill=csacd},
			nodes near coords, 
			point meta=explicit symbolic, 
			]table[
			x expr=\coordindex,
			y=sacd,       
			] {\datatable};%
			
			\addplot [
			red,
			nodes near coords, 
			forget plot,
			only marks,
			point meta=explicit symbolic, 
			every node near coord/.style={anchor=-180} 
			]table[
			x expr=\coordindex,
			y=out,     
			meta index = 8,
			] {\datatable};%
			
			\legend{FMU,FHALS,BCDP,CDTF,GCD,SaCD}%
			\end{axis}%
			\end{tikzpicture}%
	}}%
	\subfloat{
		\pgfplotstableread[col sep=space]{movierr.dat}\datatable%
		\resizebox {.35\columnwidth} {!} {
			\begin{tikzpicture}
			\begin{axis}[
			title = {Movielens},
			xticklabels from table={movierr.dat}{xaxis},
			xlabel = {(f) Rank},
			xtick=data,
			legend pos=north west,
			ylabel={Runtime in secs},
			ymode = log,
			]%
			
			\addplot [
			white,
			nodes near coords, 
			forget plot,
			]table[
			x expr=\coordindex,
			y=dum,     
			]{\datatable};%

			\addplot [
			cfmu,
			thick,
			mark=*,
			mark options={fill=cfmu},
			nodes near coords, 
			point meta=explicit symbolic, 
			]table[
			x expr=\coordindex,
			y=fmu,       
			] {\datatable};%
			
			\addplot [
			cfhals,
			thick,
			mark=*,
			mark options={fill=cfhals},
			nodes near coords, 
			point meta=explicit symbolic, 
			]table[
			x expr=\coordindex,
			y=fhals,       
			] {\datatable};%
			
			\addplot [
			cbcdp,
			thick,
			mark=*,
			mark options={fill=cbcdp},
			nodes near coords, 
			point meta=explicit symbolic, 
			]table[
			x expr=\coordindex,
			y=bcdp,       
			] {\datatable};%

			\addplot [
			ccdtf,
			thick,
			mark=*,
			mark options={fill=ccdtf},
			nodes near coords, 
			point meta=explicit symbolic, 
			]table[
			x expr=\coordindex,
			y=cdtf,       
			] {\datatable};%
			\addplot [
			cgcd,
			thick,
			mark=*,
			mark options={fill=cgcd},
			nodes near coords, 
			point meta=explicit symbolic, 
			]table[
			x expr=\coordindex,
			y=gcd,       
			] {\datatable};%
			
			\addplot [
			csacd,
			thick,
			mark=*,
			mark options={fill=csacd},
			nodes near coords, 
			point meta=explicit symbolic, 
			]table[
			x expr=\coordindex,
			y=sacd,       
			] {\datatable};%
			
			\addplot [
			red,
			nodes near coords, 
			forget plot,
			only marks,
			point meta=explicit symbolic, 
			every node near coord/.style={anchor=-180} 
			]table[
			x expr=\coordindex,
			y=out,     
			meta index = 8,
			] {\datatable};%
			
			\legend{FMU,FHALS,BCDP,CDTF,GCD,SaCD}%
			\node (source) at (axis cs:4,1447){};%
			\node (destination) at (axis cs:4,15.62){};%
			\draw[black, <->](source) -- node [below, midway] {$92\times$} (destination);%
			\node (source) at (axis cs:3.6,184){};%
			\node (destination) at (axis cs:3.6,11.62){};%
			\draw[black, <->](source) -- node [below, midway] {$11\times$} (destination);%
			\end{axis}%
			\end{tikzpicture}%
	}}%
\qquad
\subfloat{
	\pgfplotstableread[col sep=space]{gowallarmr.dat}\datatable%
	\resizebox {.35\columnwidth} {!} {
		\begin{tikzpicture}
		\begin{axis}[
		title = {Gowalla},
		y tick label style = {/pgf/number format/.cd,precision = 6,/tikz/.cd},
		xticklabels from table={gowallarmr.dat}{xaxis},
		xlabel = {(g) Rank},
		xtick=data,
		legend pos=south west,
		ylabel={RMSE}
		]%
		
		\addplot [
		white,
		nodes near coords, 
		forget plot,
		]table[
		x expr=\coordindex,
		y=dum,     
		]{\datatable};%

		\addplot [
		cfmu,
		thick,
		mark=*,
		mark options={fill=cfmu},
		nodes near coords, 
		point meta=explicit symbolic, 
		]table[
		x expr=\coordindex,
		y=fmu,       
		] {\datatable};%
		
		\addplot [
		cfhals,
		thick,
		mark=*,
		mark options={fill=cfhals},
		nodes near coords, 
		point meta=explicit symbolic, 
		]table[
		x expr=\coordindex,
		y=fhals,       
		] {\datatable};%
		
		\addplot [
		cbcdp,
		thick,
		mark=*,
		mark options={fill=cbcdp},
		nodes near coords, 
		point meta=explicit symbolic, 
		]table[
		x expr=\coordindex,
		y=bcdp,       
		] {\datatable};%
		\addplot [
		ccdtf,
		thick,
		mark=*,
		mark options={fill=ccdtf},
		nodes near coords, 
		point meta=explicit symbolic, 
		]table[
		x expr=\coordindex,
		y=cdtf,       
		] {\datatable};%
		
		\addplot [
		cgcd,
		thick,
		mark=*,
		mark options={fill=cgcd},
		nodes near coords, 
		point meta=explicit symbolic, 
		]table[
		x expr=\coordindex,
		y=gcd,       
		] {\datatable};%
		
		\addplot [
		csacd,
		thick,
		mark=*,
		mark options={fill=csacd},
		nodes near coords, 
		point meta=explicit symbolic, 
		]table[
		x expr=\coordindex,
		y=sacd,       
		] {\datatable};%
		
		\addplot [
		red,
		nodes near coords, 
		forget plot,
		only marks,
		point meta=explicit symbolic, 
		every node near coord/.style={anchor=-180} 
		]
		coordinates {(1, 0.999943)[o.o.t]};%
		
		\addplot [
		red,
		nodes near coords, 
		forget plot,
		only marks,
		point meta=explicit symbolic, 
		every node near coord/.style={anchor=-180} 
		]table[
		x expr=\coordindex,
		y=out,     
		meta index = 9,
		] {\datatable};%
		
		\legend{FMU,FHALS,BCDP,CDTF,GCD (o.o.t),SaCD}%
		\end{axis}%
		\end{tikzpicture}%
}}%
\subfloat{
	\pgfplotstableread[col sep=space]{gowallarr.dat}\datatable%
	\resizebox {.35\columnwidth} {!} {
		\begin{tikzpicture}
		\begin{axis}[
		title = {Gowalla},
		xticklabels from table={gowallarr.dat}{xaxis},
		xlabel = {(h) Rank},
		xtick=data,
		legend pos=south east,
		ylabel={Runtime in secs},
		ymode = log,
		]%
		
		\addplot [
		white,
		nodes near coords, 
		forget plot,
		]table[
		x expr=\coordindex,
		y=dum,     
		]{\datatable};%

		\addplot [
		cfmu,
		thick,
		mark=*,
		mark options={fill=cfmu},
		nodes near coords, 
		point meta=explicit symbolic, 
		]table[
		x expr=\coordindex,
		y=fmu,       
		] {\datatable};%
		
		\addplot [
		cfhals,
		thick,
		mark=*,
		mark options={fill=cfhals},
		nodes near coords, 
		point meta=explicit symbolic, 
		]table[
		x expr=\coordindex,
		y=fhals,       
		] {\datatable};%
		
		\addplot [
		cbcdp,
		thick,
		mark=*,
		mark options={fill=cbcdp},
		nodes near coords, 
		point meta=explicit symbolic, 
		]table[
		x expr=\coordindex,
		y=bcdp,       
		] {\datatable};%
		
		\addplot [
		ccdtf,
		thick,
		mark=*,
		mark options={fill=ccdtf},
		nodes near coords, 
		point meta=explicit symbolic, 
		]table[
		x expr=\coordindex,
		y=cdtf,       
		] {\datatable};%
		
		\addplot [
		cgcd,
		thick,
		mark=*,
		mark options={fill=cgcd},
		nodes near coords, 
		point meta=explicit symbolic, 
		]table[
		x expr=\coordindex,
		y=gcd,       
		] {\datatable};%
		
		\addplot [
		csacd,
		thick,
		mark=*,
		mark options={fill=csacd},
		nodes near coords, 
		point meta=explicit symbolic, 
		]table[
		x expr=\coordindex,
		y=sacd,       
		] {\datatable};%
		
		\addplot [
		red,
		nodes near coords, 
		forget plot,
		only marks,
		point meta=explicit symbolic, 
		every node near coord/.style={anchor=-180} 
		]
		coordinates {(1, 19000)[o.o.t]};%
		
		\addplot [
		red,
		nodes near coords, 
		forget plot,
		only marks,
		point meta=explicit symbolic, 
		every node near coord/.style={anchor=-180} 
		]
		coordinates {(1, 4300)[o.m]};%
		
		\addplot [
		red,
		nodes near coords, 
		forget plot,
		only marks,
		point meta=explicit symbolic, 
		every node near coord/.style={anchor=-180} 
		]table[
		x expr=\coordindex,
		y=out,     
		meta index = 9,
		] {\datatable};%
		
		\legend{FMU,FHALS,BCDP,CDTF,GCD (o.o.t),SaCD}%
		\end{axis}%
		\end{tikzpicture}%
}}%
	\caption{Prediction and Runtime performance on real-world datasets. APG does not scale for these datasets while GCD runs out of time for Gowalla dataset. [o.m. out of memory. o.o.t. out of time]}%
	\label{fig:3}%
\end{figure}
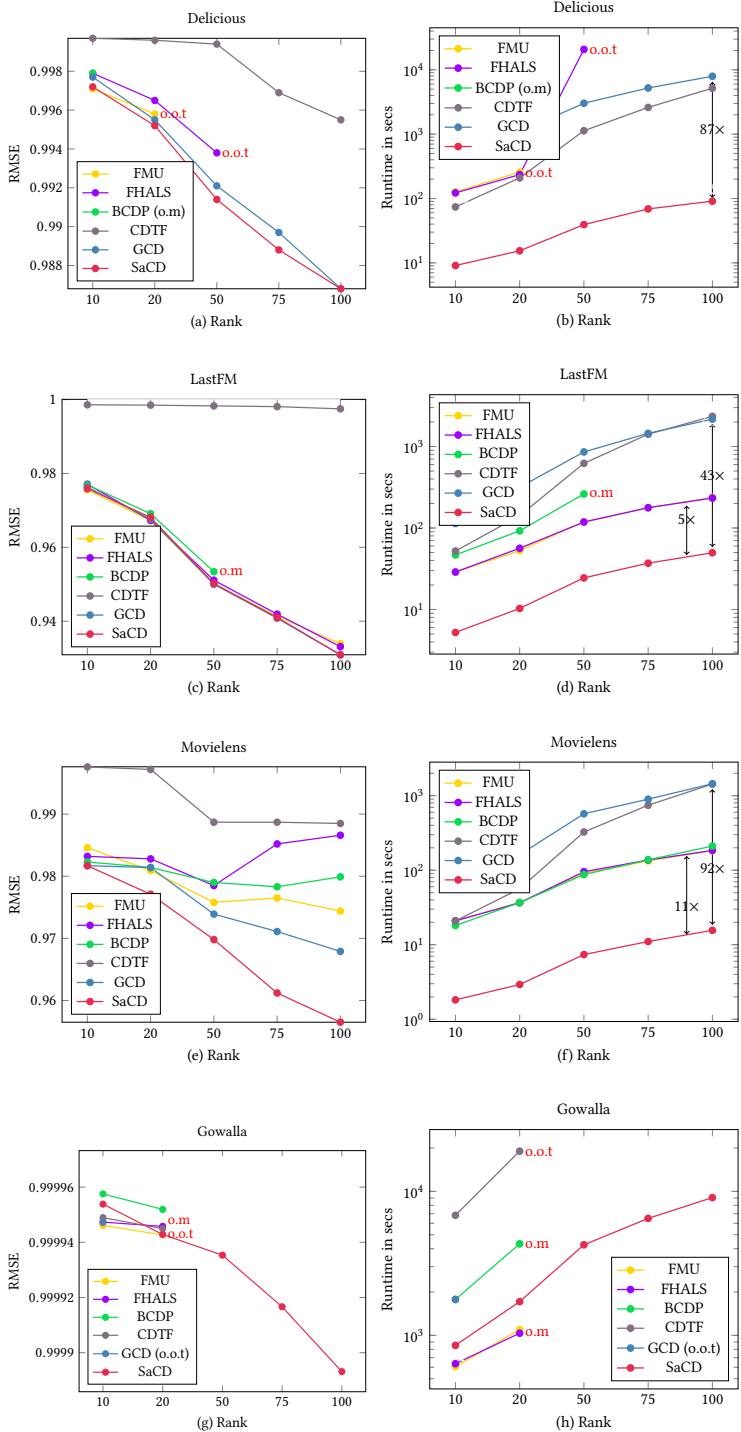

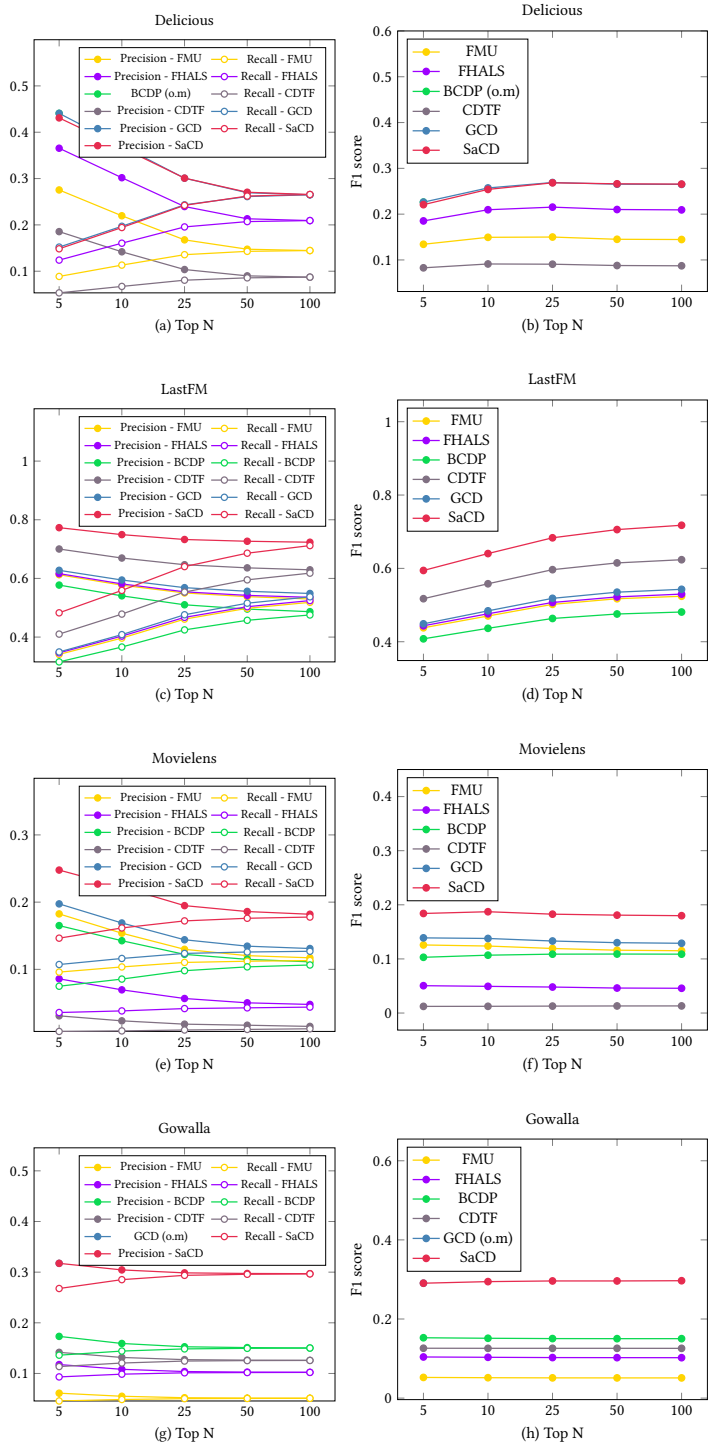
\begin{figure*}
	\centering
	\subfloat{
		\pgfplotstableread[col sep=space]{delpr.dat}\datatable%
		\resizebox {.32\columnwidth} {!} {
			\begin{tikzpicture}
			\begin{axis}[
			title = {Delicious},
			legend columns=6,transpose legend,
			enlarge y limits={upper,value=0.1},
			y tick label style = {/pgf/number format/.cd,precision = 2,/tikz/.cd},
			xticklabels from table={delpr.dat}{xaxis},
			xlabel = {(a) Top N},
			xtick=data,
			legend style={font=\fontsize{8}{8}\selectfont},
			legend pos=north east
			]%

			\addplot [
			cfmu,
			thick,
			mark=*,
			mark options={fill=cfmu},
			nodes near coords, 
			point meta=explicit symbolic, 
			]table[
			x expr=\coordindex,
			y=fmu,       
			] {\datatable};%
			
			\addplot [
			cfhals,
			thick,
			mark=*,
			mark options={fill=cfhals},
			nodes near coords, 
			point meta=explicit symbolic, 
			]table[
			x expr=\coordindex,
			y=fhals,       
			] {\datatable};%
			
			\addplot [
			cbcdp,
			thick,
			mark=*,
			mark options={fill=cbcdp},
			nodes near coords, 
			point meta=explicit symbolic, 
			]table[
			x expr=\coordindex,
			y=bcdp,       
			] {\datatable};%
			
			\addplot [
			ccdtf,
			thick,
			mark=*,
			mark options={fill=ccdtf},
			nodes near coords, 
			point meta=explicit symbolic, 
			]table[
			x expr=\coordindex,
			y=cdtf,       
			] {\datatable};%
			
			\addplot [
			cgcd,
			thick,
			mark=*,
			mark options={fill=cgcd},
			nodes near coords, 
			point meta=explicit symbolic, 
			]table[
			x expr=\coordindex,
			y=gcd,       
			] {\datatable};%
			
			\addplot [
			csacd,
			thick,
			mark=*,
			mark options={fill=csacd},
			nodes near coords, 
			point meta=explicit symbolic, 
			]table[
			x expr=\coordindex,
			y=sacd,       
			] {\datatable};%
			
			\addplot [
			cfmu,
			thick,
			mark=*,
			mark options={fill=white},
			nodes near coords, 
			point meta=explicit symbolic, 
			]table[
			x expr=\coordindex,
			y=fmur,       
			] {\datatable};%
			
			\addplot [
			cfhals,
			thick,
			mark=*,
			mark options={fill=white},
			nodes near coords, 
			point meta=explicit symbolic, 
			]table[
			x expr=\coordindex,
			y=fhalsr,       
			] {\datatable};%
			
		    \addplot [
			ccdtf,
			thick,
			mark=*,
			mark options={fill=white},
			nodes near coords, 
			point meta=explicit symbolic, 
			]table[
			x expr=\coordindex,
			y=cdtfr,       
			] {\datatable};%
			
			\addplot [
			cgcd,
			thick,
			mark=*,
			mark options={fill=white},
			nodes near coords, 
			point meta=explicit symbolic, 
			]table[
			x expr=\coordindex,
			y=gcdr,       
			] {\datatable};%
			
			\addplot [
			csacd,
			thick,
			mark=*,
			mark options={fill=white},
			nodes near coords, 
			point meta=explicit symbolic, 
			]table[
			x expr=\coordindex,
			y=sacdr,       
			] {\datatable};%
			\addplot [
			white,
			nodes near coords, 
			forget plot,
			]table[
			x expr=\coordindex,
			y=dum,     
			]{\datatable};%
			
			\legend{Precision - FMU,Precision - FHALS,BCDP (o.m),Precision - CDTF,Precision - GCD,Precision - SaCD,Recall - FMU,Recall - FHALS,Recall - CDTF,Recall - GCD,Recall - SaCD}%
			\end{axis}%
			\end{tikzpicture}%
	}}
	\subfloat{
		\pgfplotstableread[col sep=space]{delf1.dat}\datatable%
		\resizebox {.35\columnwidth} {!} {
			\begin{tikzpicture}
			\begin{axis}[
			title = {Delicious},
			xticklabels from table={delf1.dat}{xaxis},
			xlabel = {(b) Top N},
			xtick=data,
			ymax = 0.6,
			legend pos=north west,
			ylabel={F1 score}
			]%
			
			\addplot [
			white,
			nodes near coords, 
			forget plot,
			]table[
			x expr=\coordindex,
			y=dum,     
			]{\datatable};%

			\addplot [
			cfmu,
			thick,
			mark=*,
			mark options={fill=cfmu},
			nodes near coords, 
			point meta=explicit symbolic, 
			]table[
			x expr=\coordindex,
			y=fmu,       
			] {\datatable};%
			
			\addplot [
			cfhals,
			thick,
			mark=*,
			mark options={fill=cfhals},
			nodes near coords, 
			point meta=explicit symbolic, 
			]table[
			x expr=\coordindex,
			y=fhals,       
			] {\datatable};%
			
			\addplot [
			cbcdp,
			thick,
			mark=*,
			mark options={fill=cbcdp},
			nodes near coords, 
			point meta=explicit symbolic, 
			]table[
			x expr=\coordindex,
			y=bcdp,       
			] {\datatable};%
			
			\addplot [
			ccdtf,
			thick,
			mark=*,
			mark options={fill=ccdtf},
			nodes near coords, 
			point meta=explicit symbolic, 
			]table[
			x expr=\coordindex,
			y=cdtf,       
			] {\datatable};%
			
			\addplot [
			cgcd,
			thick,
			mark=*,
			mark options={fill=cgcd},
			nodes near coords, 
			point meta=explicit symbolic, 
			]table[
			x expr=\coordindex,
			y=gcd,       
			] {\datatable};%
			
			\addplot [
			csacd,
			thick,
			mark=*,
			mark options={fill=csacd},
			nodes near coords, 
			point meta=explicit symbolic, 
			]table[
			x expr=\coordindex,
			y=sacd,       
			] {\datatable};%
			
			
			\legend{FMU,FHALS,BCDP (o.m),CDTF,GCD,SaCD}%
			\end{axis}%
			\end{tikzpicture}%
	}}
	\qquad
	\subfloat{
		\pgfplotstableread[col sep=space]{lastfmpr.dat}\datatable%
		\resizebox {.32\columnwidth} {!} {
			\begin{tikzpicture}
			\begin{axis}[
			title = {LastFM},
			enlarge y limits={upper,value=0.1},
			y tick label style = {/pgf/number format/.cd,precision = 2,/tikz/.cd},
			legend columns=6,transpose legend,
			legend entries={Precision - FMU,Precision - FHALS, Precision - BCDP,Precision - CDTF,Precision - GCD,Precision - SaCD,Recall - FMU,Recall - FHALS,Recall - BCDP,Recall - CDTF,Recall - GCD,Recall - SaCD},
			xticklabels from table={lastfmpr.dat}{xaxis},
			xlabel = {(c) Top N},
			xtick=data,
			legend style={font=\fontsize{8}{8}\selectfont},
			legend pos=north east
			]%

			\addplot [
			cfmu,
			thick,
			mark=*,
			mark options={fill=cfmu},
			nodes near coords, 
			point meta=explicit symbolic, 
			]table[
			x expr=\coordindex,
			y=fmu,       
			] {\datatable};%
			
			\addplot [
			cfhals,
			thick,
			mark=*,
			mark options={fill=cfhals},
			nodes near coords, 
			point meta=explicit symbolic, 
			]table[
			x expr=\coordindex,
			y=fhals,       
			] {\datatable};%
			
			\addplot [
			cbcdp,
			thick,
			mark=*,
			mark options={fill=cbcdp},
			nodes near coords, 
			point meta=explicit symbolic, 
			]table[
			x expr=\coordindex,
			y=bcdp,       
			] {\datatable};%
			
			\addplot [
			ccdtf,
			thick,
			mark=*,
			mark options={fill=ccdtf},
			nodes near coords, 
			point meta=explicit symbolic, 
			]table[
			x expr=\coordindex,
			y=cdtf,       
			] {\datatable};%
			
			\addplot [
			cgcd,
			thick,
			mark=*,
			mark options={fill=cgcd},
			nodes near coords, 
			point meta=explicit symbolic, 
			]table[
			x expr=\coordindex,
			y=gcd,       
			] {\datatable};%

			\addplot [
			csacd,
			thick,
			mark=*,
			mark options={fill=csacd},
			nodes near coords, 
			point meta=explicit symbolic, 
			]table[
			x expr=\coordindex,
			y=sacd,       
			] {\datatable};%
			
			\addplot [
			cfmu,
			thick,
			mark=*,
			mark options={fill=white},
			nodes near coords, 
			point meta=explicit symbolic, 
			]table[
			x expr=\coordindex,
			y=fmur,       
			] {\datatable};%
			
			\addplot [
			cfhals,
			thick,
			mark=*,
			mark options={fill=white},
			nodes near coords, 
			point meta=explicit symbolic, 
			]table[
			x expr=\coordindex,
			y=fhalsr,       
			] {\datatable};%
			
			\addplot [
			cbcdp,
			thick,
			mark=*,
			mark options={fill=white},
			nodes near coords, 
			point meta=explicit symbolic, 
			]table[
			x expr=\coordindex,
			y=bcdpr,       
			] {\datatable};%
			
			\addplot [
			ccdtf,
			thick,
			mark=*,
			mark options={fill=white},
			nodes near coords, 
			point meta=explicit symbolic, 
			]table[
			x expr=\coordindex,
			y=cdtfr,       
			] {\datatable};%
			
			\addplot [
			cgcd,
			thick,
			mark=*,
			mark options={fill=white},
			nodes near coords, 
			point meta=explicit symbolic, 
			]table[
			x expr=\coordindex,
			y=gcdr,       
			] {\datatable};%
			
			\addplot [
			csacd,
			thick,
			mark=*,
			mark options={fill=white},
			nodes near coords, 
			point meta=explicit symbolic, 
			]table[
			x expr=\coordindex,
			y=sacdr,       
			] {\datatable};%
			\addplot [
			white,
			nodes near coords, 
			forget plot,
			]table[
			x expr=\coordindex,
			y=dum,     
			]{\datatable};%
			
			\end{axis}%
			\end{tikzpicture}%
	}}
	\subfloat{
		\pgfplotstableread[col sep=space]{lastf1.dat}\datatable%
		\resizebox {.35\columnwidth} {!} {
			\begin{tikzpicture}
			\begin{axis}[
			title = {LastFM},
			xticklabels from table={lastf1.dat}{xaxis},
			xlabel = {(d) Top N},
			xtick=data,
			legend pos=north west,
			ylabel={F1 score}
			]%
			
			\addplot [
			white,
			nodes near coords, 
			forget plot,
			]table[
			x expr=\coordindex,
			y=dum,     
			]{\datatable};%

			\addplot [
			cfmu,
			thick,
			mark=*,
			mark options={fill=cfmu},
			nodes near coords, 
			point meta=explicit symbolic, 
			]table[
			x expr=\coordindex,
			y=fmu,       
			] {\datatable};%
			
			\addplot [
			cfhals,
			thick,
			mark=*,
			mark options={fill=cfhals},
			nodes near coords, 
			point meta=explicit symbolic, 
			]table[
			x expr=\coordindex,
			y=fhals,       
			] {\datatable};%
			
			\addplot [
			cbcdp,
			thick,
			mark=*,
			mark options={fill=cbcdp},
			nodes near coords, 
			point meta=explicit symbolic, 
			]table[
			x expr=\coordindex,
			y=bcdp,       
			] {\datatable};%
			
			\addplot [
			ccdtf,
			thick,
			mark=*,
			mark options={fill=ccdtf},
			nodes near coords, 
			point meta=explicit symbolic, 
			]table[
			x expr=\coordindex,
			y=cdtf,       
			] {\datatable};%
			
			\addplot [
			cgcd,
			thick,
			mark=*,
			mark options={fill=cgcd},
			nodes near coords, 
			point meta=explicit symbolic, 
			]table[
			x expr=\coordindex,
			y=gcd,       
			] {\datatable};%
			
			\addplot [
			csacd,
			thick,
			mark=*,
			mark options={fill=csacd},
			nodes near coords, 
			point meta=explicit symbolic, 
			]table[
			x expr=\coordindex,
			y=sacd,       
			] {\datatable};%
			
			
			\legend{FMU,FHALS,BCDP,CDTF,GCD,SaCD}%
			\end{axis}%
			\end{tikzpicture}%
		}}
	\quad
	\subfloat{
		\pgfplotstableread[col sep=space]{moviepr.dat}\datatable%
		\resizebox {.32\columnwidth} {!} {
			\begin{tikzpicture}
			\begin{axis}[
			title = {Movielens},
			legend columns=6,transpose legend,
			enlarge y limits={upper,value=0.1},
			y tick label style = {/pgf/number format/.cd,precision = 2,/tikz/.cd},
			xticklabels from table={moviepr.dat}{xaxis},
			xlabel = {(e) Top N},
			xtick=data,
			ymax = 0.35,
			legend style={font=\fontsize{8}{8}\selectfont},
			legend pos=north east
			]%

			\addplot [
			cfmu,
			thick,
			mark=*,
			mark options={fill=cfmu},
			nodes near coords, 
			point meta=explicit symbolic, 
			]table[
			x expr=\coordindex,
			y=fmu,       
			] {\datatable};%
			
			\addplot [
			cfhals,
			thick,
			mark=*,
			mark options={fill=cfhals},
			nodes near coords, 
			point meta=explicit symbolic, 
			]table[
			x expr=\coordindex,
			y=fhals,       
			] {\datatable};%
			
			\addplot [
			cbcdp,
			thick,
			mark=*,
			mark options={fill=cbcdp},
			nodes near coords, 
			point meta=explicit symbolic, 
			]table[
			x expr=\coordindex,
			y=bcdp,       
			] {\datatable};%
			
		    \addplot [
			ccdtf,
			thick,
			mark=*,
			mark options={fill=ccdtf},
			nodes near coords, 
			point meta=explicit symbolic, 
			]table[
			x expr=\coordindex,
			y=cdtf,       
			] {\datatable};%
			
			\addplot [
			cgcd,
			thick,
			mark=*,
			mark options={fill=cgcd},
			nodes near coords, 
			point meta=explicit symbolic, 
			]table[
			x expr=\coordindex,
			y=gcd,       
			] {\datatable};%
			
			\addplot [
			csacd,
			thick,
			mark=*,
			mark options={fill=csacd},
			nodes near coords, 
			point meta=explicit symbolic, 
			]table[
			x expr=\coordindex,
			y=sacd,       
			] {\datatable};%
			
			\addplot [
			cfmu,
			thick,
			mark=*,
			mark options={fill=white},
			nodes near coords, 
			point meta=explicit symbolic, 
			]table[
			x expr=\coordindex,
			y=fmur,       
			] {\datatable};%
			
			\addplot [
			cfhals,
			thick,
			mark=*,
			mark options={fill=white},
			nodes near coords, 
			point meta=explicit symbolic, 
			]table[
			x expr=\coordindex,
			y=fhalsr,       
			] {\datatable};%
			
			\addplot [
			cbcdp,
			thick,
			mark=*,
			mark options={fill=white},
			nodes near coords, 
			point meta=explicit symbolic, 
			]table[
			x expr=\coordindex,
			y=bcdpr,       
			] {\datatable};%
			
			\addplot [
			ccdtf,
			thick,
			mark=*,
			mark options={fill=white},
			nodes near coords, 
			point meta=explicit symbolic, 
			]table[
			x expr=\coordindex,
			y=cdtfr,       
			] {\datatable};%
			
			\addplot [
			cgcd,
			thick,
			mark=*,
			mark options={fill=white},
			nodes near coords, 
			point meta=explicit symbolic, 
			]table[
			x expr=\coordindex,
			y=gcdr,       
			] {\datatable};%
			
			\addplot [
			csacd,
			thick,
			mark=*,
			mark options={fill=white},
			nodes near coords, 
			point meta=explicit symbolic, 
			]table[
			x expr=\coordindex,
			y=sacdr,       
			] {\datatable};%
			\addplot [
			white,
			nodes near coords, 
			forget plot,
			]table[
			x expr=\coordindex,
			y=dum,     
			]{\datatable};%
			
			\legend{Precision - FMU,Precision - FHALS,Precision - BCDP,Precision - CDTF,Precision - GCD,Precision - SaCD,Recall - FMU,Recall - FHALS,Recall - BCDP,Recall - CDTF,Recall - GCD,Recall - SaCD}%
			\end{axis}%
			\end{tikzpicture}%
	}}
	\subfloat{
		\pgfplotstableread[col sep=space]{movief1.dat}\datatable%
		\resizebox {.35\columnwidth} {!} {
			\begin{tikzpicture}
			\begin{axis}[
			title = {Movielens},
			xticklabels from table={movief1.dat}{xaxis},
			xlabel = {(f) Top N},
			ymax = 0.45,
			xtick=data,
			legend pos=north west,
			ylabel={F1 score}
			]%
			
			\addplot [
			white,
			nodes near coords, 
			forget plot,
			]table[
			x expr=\coordindex,
			y=dum,     
			]{\datatable};%

			\addplot [
			cfmu,
			thick,
			mark=*,
			mark options={fill=cfmu},
			nodes near coords, 
			point meta=explicit symbolic, 
			]table[
			x expr=\coordindex,
			y=fmu,       
			] {\datatable};%
			
			\addplot [
			cfhals,
			thick,
			mark=*,
			mark options={fill=cfhals},
			nodes near coords, 
			point meta=explicit symbolic, 
			]table[
			x expr=\coordindex,
			y=fhals,       
			] {\datatable};%
			
			\addplot [
			cbcdp,
			thick,
			mark=*,
			mark options={fill=cbcdp},
			nodes near coords, 
			point meta=explicit symbolic, 
			]table[
			x expr=\coordindex,
			y=bcdp,       
			] {\datatable};%
			
			\addplot [
    		ccdtf,
    		thick,
    		mark=*,
    		mark options={fill=ccdtf},
    		nodes near coords, 
    		point meta=explicit symbolic, 
    		]table[
    		x expr=\coordindex,
    		y=cdtf,       
    		] {\datatable};%
    		
			\addplot [
			cgcd,
			thick,
			mark=*,
			mark options={fill=cgcd},
			nodes near coords, 
			point meta=explicit symbolic, 
			]table[
			x expr=\coordindex,
			y=gcd,       
			] {\datatable};%
			
			\addplot [
			csacd,
			thick,
			mark=*,
			mark options={fill=csacd},
			nodes near coords, 
			point meta=explicit symbolic, 
			]table[
			x expr=\coordindex,
			y=sacd,       
			] {\datatable};%
			
			
			\legend{FMU,FHALS,BCDP,CDTF,GCD,SaCD}%
			\end{axis}%
			\end{tikzpicture}%
	}}%
\qquad
\subfloat{
	\pgfplotstableread[col sep=space]{gowallapr.dat}\datatable%
	\resizebox {.32\columnwidth} {!} {
		\begin{tikzpicture}
		\begin{axis}[
		title = {Gowalla},
		legend columns=6,transpose legend,
		enlarge y limits={upper,value=0.1},
		y tick label style = {/pgf/number format/.cd,precision = 2,/tikz/.cd},
		xticklabels from table={gowallapr.dat}{xaxis},
		xlabel = {(g) Top N},
		xtick=data,
		legend style={font=\fontsize{8}{8}\selectfont},
		legend pos=north east
		]%

		\addplot [
		cfmu,
		thick,
		mark=*,
		mark options={fill=cfmu},
		nodes near coords, 
		point meta=explicit symbolic, 
		]table[
		x expr=\coordindex,
		y=fmu,       
		] {\datatable};%
		
		\addplot [
		cfhals,
		thick,
		mark=*,
		mark options={fill=cfhals},
		nodes near coords, 
		point meta=explicit symbolic, 
		]table[
		x expr=\coordindex,
		y=fhals,       
		] {\datatable};%
		
		\addplot [
		cbcdp,
		thick,
		mark=*,
		mark options={fill=cbcdp},
		nodes near coords, 
		point meta=explicit symbolic, 
		]table[
		x expr=\coordindex,
		y=bcdp,       
		] {\datatable};%
		
		\addplot [
		ccdtf,
		thick,
		mark=*,
		mark options={fill=ccdtf},
		nodes near coords, 
		point meta=explicit symbolic, 
		]table[
		x expr=\coordindex,
		y=cdtf,       
		] {\datatable};%
		
		\addplot [
		cgcd,
		thick,
		mark=*,
		mark options={fill=cgcd},
		nodes near coords, 
		point meta=explicit symbolic, 
		]table[
		x expr=\coordindex,
		y=gcd,       
		] {\datatable};%
		
		\addplot [
		csacd,
		thick,
		mark=*,
		mark options={fill=csacd},
		nodes near coords, 
		point meta=explicit symbolic, 
		]table[
		x expr=\coordindex,
		y=sacd,       
		] {\datatable};%
		
		\addplot [
		cfmu,
		thick,
		mark=*,
		mark options={fill=white},
		nodes near coords, 
		point meta=explicit symbolic, 
		]table[
		x expr=\coordindex,
		y=fmur,       
		] {\datatable};%
		
		\addplot [
		cfhals,
		thick,
		mark=*,
		mark options={fill=white},
		nodes near coords, 
		point meta=explicit symbolic, 
		]table[
		x expr=\coordindex,
		y=fhalsr,       
		] {\datatable};%
		
			\addplot [
			cbcdp,
			thick,
			mark=*,
			mark options={fill=white},
			nodes near coords, 
			point meta=explicit symbolic, 
			]table[
			x expr=\coordindex,
			y=bcdpr,       
			] {\datatable};%
		
		\addplot [
		ccdtf,
		thick,
		mark=*,
		mark options={fill=white},
		nodes near coords, 
		point meta=explicit symbolic, 
		]table[
		x expr=\coordindex,
		y=cdtfr,       
		]{\datatable};%
		
		\addplot [
		csacd,
		thick,
		mark=*,
		mark options={fill=white},
		nodes near coords, 
		point meta=explicit symbolic, 
		]table[
		x expr=\coordindex,
		y=sacdr,       
		] {\datatable};%
		\addplot [
		white,
		nodes near coords, 
		forget plot,
		]table[
		x expr=\coordindex,
		y=dum,     
		]{\datatable};%
		
		\legend{Precision - FMU,Precision - FHALS,Precision - BCDP,Precision - CDTF,GCD (o.m),Precision - SaCD,Recall - FMU,Recall - FHALS,Recall - BCDP,Recall - CDTF,Recall - SaCD}%
		\end{axis}%
		\end{tikzpicture}%
}}
\subfloat{
	\pgfplotstableread[col sep=space]{gowallaf1.dat}\datatable%
	\resizebox {.35\columnwidth} {!} {
		\begin{tikzpicture}
		\begin{axis}[
		title = {Gowalla},
		xticklabels from table={gowallaf1.dat}{xaxis},
		xlabel = {(h) Top N},
		xtick=data,
		legend pos=north west,
		ylabel={F1 score}
		]%
		
		\addplot [
		white,
		nodes near coords, 
		forget plot,
		]table[
		x expr=\coordindex,
		y=dum,     
		]{\datatable};%

		\addplot [
		cfmu,
		thick,
		mark=*,
		mark options={fill=cfmu},
		nodes near coords, 
		point meta=explicit symbolic, 
		]table[
		x expr=\coordindex,
		y=fmu,       
		] {\datatable};%
		
		\addplot [
		cfhals,
		thick,
		mark=*,
		mark options={fill=cfhals},
		nodes near coords, 
		point meta=explicit symbolic, 
		]table[
		x expr=\coordindex,
		y=fhals,       
		] {\datatable};%
		
		\addplot [
		cbcdp,
		thick,
		mark=*,
		mark options={fill=cbcdp},
		nodes near coords, 
		point meta=explicit symbolic, 
		]table[
		x expr=\coordindex,
		y=bcdp,       
		] {\datatable};%
		
		\addplot [
		ccdtf,
		thick,
		mark=*,
		mark options={fill=ccdtf},
		nodes near coords, 
		point meta=explicit symbolic, 
		]table[
		x expr=\coordindex,
		y=cdtf,       
		] {\datatable};%
		
		\addplot [
		cgcd,
		thick,
		mark=*,
		mark options={fill=cgcd},
		nodes near coords, 
		point meta=explicit symbolic, 
		]table[
		x expr=\coordindex,
		y=gcd,       
		] {\datatable};%
		
		\addplot [
		csacd,
		thick,
		mark=*,
		mark options={fill=csacd},
		nodes near coords, 
		point meta=explicit symbolic, 
		]table[
		x expr=\coordindex,
		y=sacd,       
		] {\datatable};%
		
		
		\legend{FMU,FHALS,BCDP,CDTF,GCD (o.m),SaCD}%
		\end{axis}%
		\end{tikzpicture}%
}}
	\caption{Precision, Recall, and F1 score on real-world datasets.}%
	\label{fig:rec}%
\end{figure*}

\subsection{Pattern Mining}
In addition to missing value prediction, NTF can also be used to identify the patterns automatically. 

For LBSN datasets such as Gowalla, we have time as the $3^{rd}$ mode. By setting the rank of a tensor as $4$ in the factorization process, we identified $4$ different patterns in the temporal factor matrix. Table~\ref{table_gowalla} shows the PD values of all algorithms on the Gowalla dataset. SaCD outperforms all the baselines. GCD ran out of time due to the large mode length of the tensor. While FMU and FHALS can execute due to the low rank setting, they are not able to distinguish the patterns distinctly. 

In Figure~\ref{fig:temporal}, we plot the values of the factor matrix in “temporal mode”, which have $24$ $hours$ as $x-axis$ and $y-axis$ represents the normalized value of the factor matrix in each column. Figure~\ref{fig:temporal}(e) shows the patterns obtained by SaCD. The red pattern shows a peak at $1$ $pm$ and $9$ $pm$ that probably indicates the lunch and dinner time. The pink pattern shows a very common $7$ $am$ to $10$ $pm$ activity. On the other hand, the blue pattern shows a unique pattern with a peak at $1$ $am$, indicating night time activity. The green pattern shows activities between $1$ $pm$ and $6$ $am$. With a proper domain knowledge, the kind of activities that happens in different time periods can be easily interpreted by using distinct patterns. It is interesting to note that unlike FMU and CDTF, SaCD avoids simultaneous elimination problem (a state where similar patterns are derived multiple times)~\cite{zou2008f}. In Figure~\ref{fig:temporal}(a), pink and blue patterns are highly similar, and  in Figure~\ref{fig:temporal}(d), red pattern is same as the green pattern and pink pattern is same as the blue pattern. In comparison, patterns derived from SaCD are highly distinctive. 

\begin{table}[!t]
\renewcommand{\arraystretch}{1.3}
\caption{Pattern Distinctiveness ($PD$) and Runtime on the Gowalla dataset (lower values are better).}
\label{table_gowalla}
\centering
\begin{tabular}{lll}
\toprule
\bfseries Algorithm & \bfseries PD & \bfseries Runtime in secs\\
\toprule
FMU  & $0.69$   & $633.95$\\
FHALS & $0.41$  & $623.66$\\
BCDP & $0.63$ & $1076.01$\\
CDTF & $0.71$  &  $2537.64$\\
GCD & {\color{red} $o.o.t$}  &  {\color{red} $o.o.t$}\\
SaCD & $\textbf{0.34}$  &  $\textbf{363.27}$\\
\bottomrule
\end{tabular}
\end{table}


\begin{figure*}
	\centering
	\subfloat[FMU]{{\includegraphics[width=2.5in,height=2in]{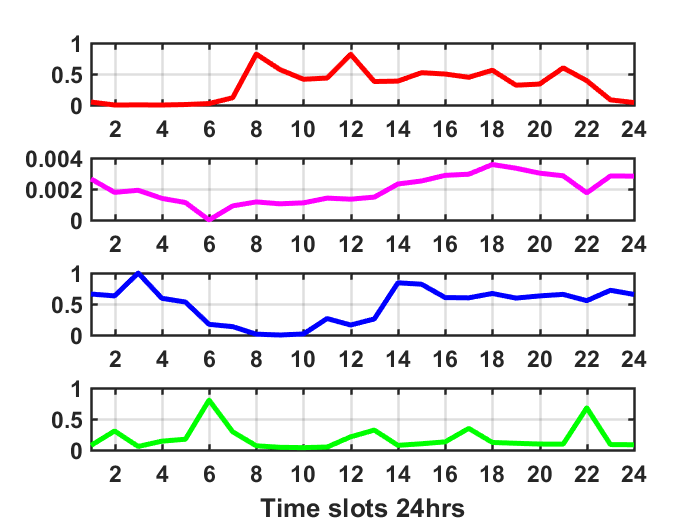}}}%
	\hspace*{5px}
	\subfloat[FHALS]{{\includegraphics[width=2.5in,height=2in]{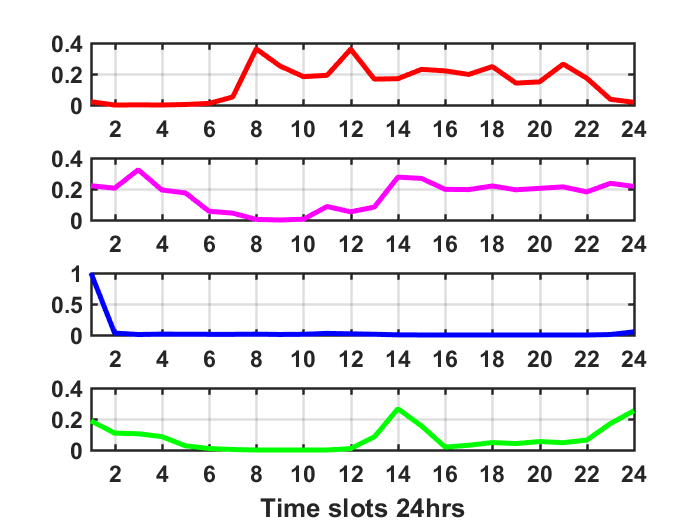}}}%
	\qquad
	\subfloat[BCDP]{{\includegraphics[width=2.5in,height=2in]{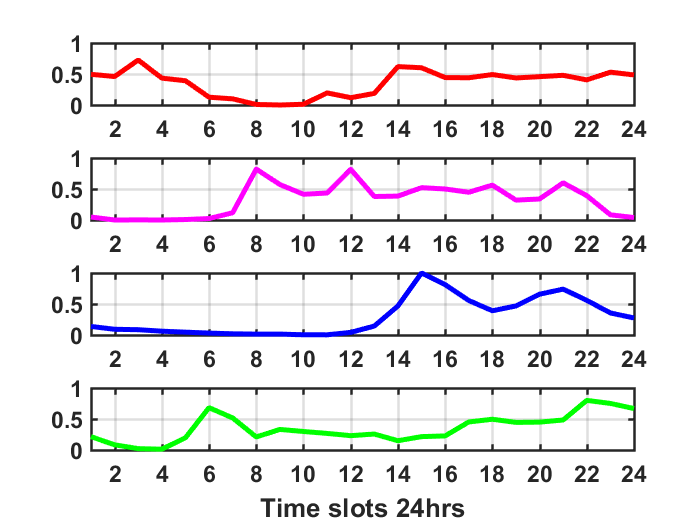}}}%
	\hspace*{5px}
	\subfloat[CDTF]{{\includegraphics[width=2.5in,height=2in]{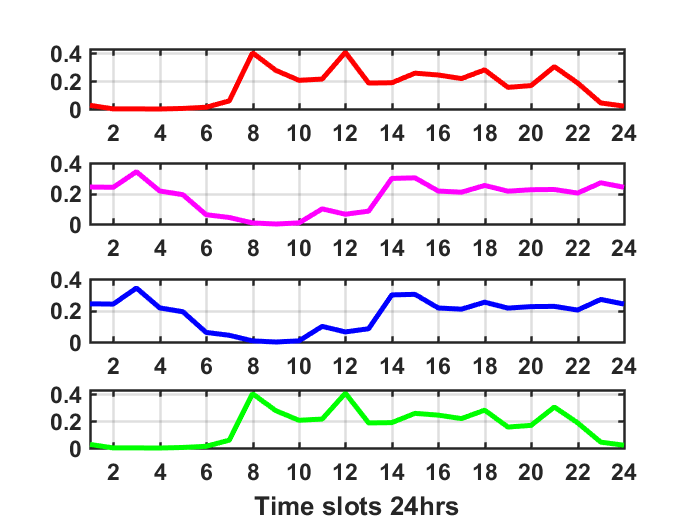}}}%
	\qquad
	\subfloat[SaCD]{{\includegraphics[width=2.5in,height=2in]{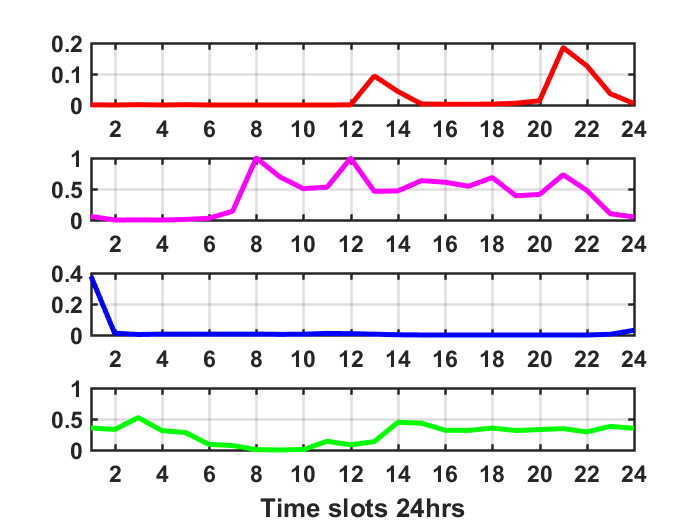}}}%
	\caption{Temporal patterns derived from the 3rd mode (time) of the tensor on the Gowalla dataset.}%
	\label{fig:temporal}
\end{figure*}

\subsection{Parallelization}
Figure~\ref{fig:2}(d) shows the time taken for $mttkrp$ and the factor matrix update without multi-core parallelization. SaCD allows to minimize the time taken for factor matrix update by avoiding the frequent gradient calculation, however, the complexity of $mttkrp$ remains the same as the traditional element selection-based CD algorithm.    
        
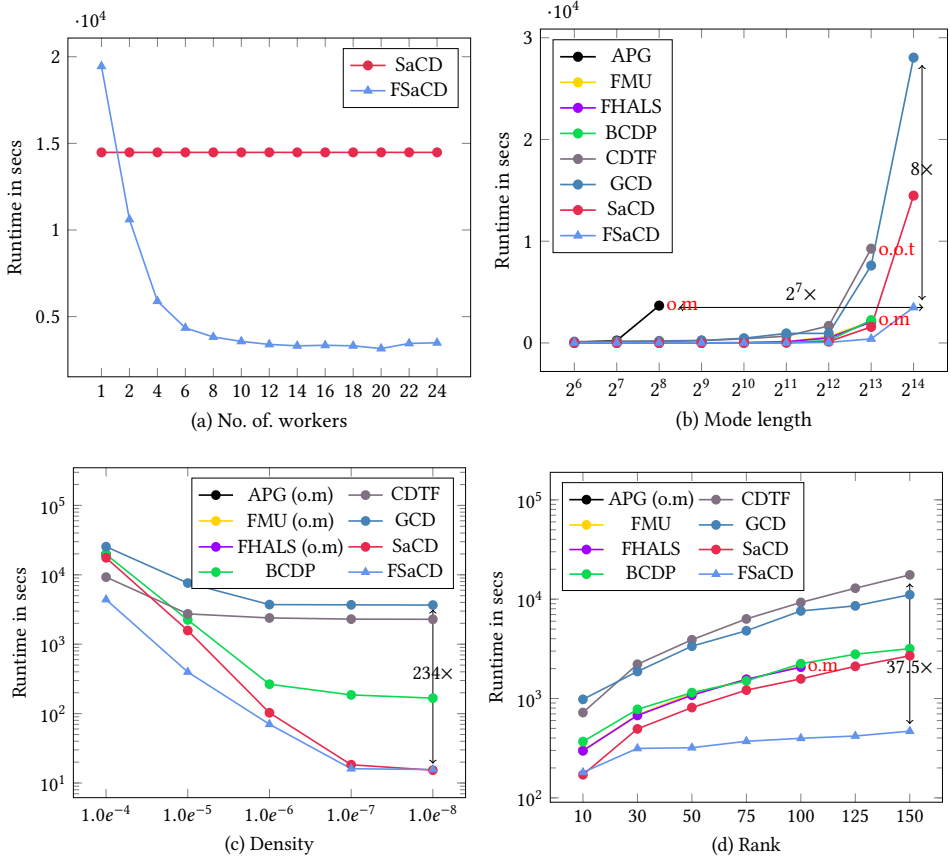
\begin{figure*}
	\centering
	\subfloat{
		\pgfplotstableread[col sep=space]{workers.dat}\datatable%
		\resizebox {.46\columnwidth} {!} {
			\begin{tikzpicture}
			\begin{axis}[
			xticklabels from table={workers.dat}{xaxis},
			xlabel = {(a) No. of. workers},
			xtick=data,
			ylabel={Runtime in secs},
			legend pos=north east
			]%
			
			\addplot[
			csacd,			
			thick,
			mark = *,
			mark options={fill=csacd},
			]table[
			x expr=\coordindex,
			y=sacd,     
			]{\datatable};%
			
			\addplot[
			bleu,
			thick,
			mark = triangle*,
			mark options = {fill=bleu},
			]table[
			x expr=\coordindex,
			y=fsacd,     
			]{\datatable};%
			\legend{SaCD,FSaCD}%
			\end{axis}
			\end{tikzpicture}			
	}}%
	\subfloat{
		\pgfplotstableread[col sep=space]{dimhpc.dat}\datatable%
		\resizebox {.45\columnwidth} {!} {
			\begin{tikzpicture}
			\begin{axis}[
			xticklabels from table={dimhpc.dat}{xaxis},
			xlabel = {(b) Mode length},
			xtick=data,
			ylabel={Runtime in secs},
			legend pos=north west
			]%
			
			\addplot [
			white,
			nodes near coords, 
			forget plot,
			]table[
			x expr=\coordindex,
			y=dum,     
			]{\datatable};%
			
			\addplot [
			capg,
			thick,
			mark=*,
			mark options={fill=capg},
			nodes near coords, 
			point meta=explicit symbolic, 
			]table[
			x expr=\coordindex,
			y=apg,       
			] {\datatable};%
			
			\addplot [
			cfmu,
			thick,
			mark=*,
			mark options={fill=cfmu},
			nodes near coords, 
			point meta=explicit symbolic, 
			]table[
			x expr=\coordindex,
			y=fmu,       
			] {\datatable};%
			
			\addplot [
			cfhals,
			thick,
			mark=*,
			mark options={fill=cfhals},
			nodes near coords, 
			point meta=explicit symbolic, 
			]table[
			x expr=\coordindex,
			y=fhals,       
			] {\datatable};%
			
			\addplot [
			cbcdp,
			thick,
			mark=*,
			mark options={fill=cbcdp},
			nodes near coords, 
			point meta=explicit symbolic, 
			]table[
			x expr=\coordindex,
			y=bcdp,       
			] {\datatable};%
						
			\addplot [
			ccdtf,
			thick,
			mark=*,
			mark options={fill=ccdtf},
			nodes near coords, 
			point meta=explicit symbolic, 
			]table[
			x expr=\coordindex,
			y=cdtf,       
			] {\datatable};%
			
			\addplot [
			cgcd,
			thick,
			mark=*,
			mark options={fill=cgcd},
			nodes near coords, 
			point meta=explicit symbolic, 
			]table[
			x expr=\coordindex,
			y=gcd,       
			] {\datatable};%
			
			\addplot [
			csacd,
			thick,
			mark=*,
			mark options={fill=csacd},
			nodes near coords, 
			point meta=explicit symbolic, 
			]table[
			x expr=\coordindex,
			y=sacd,       
			] {\datatable};%
			
			\addplot[
			bleu,
			thick,
			mark = triangle*,
			mark options = {fill=bleu},
			]table[
			x expr=\coordindex,
			y=fsacd,     
			]{\datatable};%
			
			\addplot [
			red,
			nodes near coords, 
			forget plot,
			only marks,
			point meta=explicit symbolic, 
			every node near coord/.style={anchor=-180} 
			]table[
			x expr=\coordindex,
			y=out,     
			meta index = 11,
			] {\datatable};%
			
			\addplot [
    		red,
    		nodes near coords, 
    		forget plot,
    		only marks,
    		point meta=explicit symbolic, 
    		every node near coord/.style={anchor=-180} 
    		]
    		coordinates {(7, 9200)[o.o.t]};%
			\legend{APG,FMU,FHALS,BCDP,CDTF,GCD,SaCD,FSaCD}%
			
			\node (source) at (axis cs:8.2,3490){};%
			\node (destination) at (axis cs:8.2,28050){};%
			\draw[black, <->](source) -- node [above, midway] {$8\times$} (destination);%
			\node (source) at (axis cs:2.3,3490){};%
			\node (destination) at (axis cs:8.4,3490){};%
			\draw[black, <->](source) -- node [above, midway] {$2^7\times$} (destination);%
			\end{axis}
			\end{tikzpicture}			
	}}%
	\qquad
	\subfloat{
		\pgfplotstableread[col sep=space]{denhpc.dat}\datatable%
		\resizebox {.45\columnwidth} {!} {
			\begin{tikzpicture}
			\begin{axis}[
			legend columns=4,transpose legend,
			xticklabels from table={denhpc.dat}{xaxis},
			xlabel = {(c) Density},
			xtick=data,
			ymode = log,
			ylabel={Runtime in secs},
			legend pos=north east
			]%
			
			\addplot [
			white,
			nodes near coords, 
			forget plot,
			]table[
			x expr=\coordindex,
			y=dum,     
			]{\datatable};%
			
			\addplot [
			capg,
			thick,
			mark=*,
			mark options={fill=capg},
			nodes near coords, 
			point meta=explicit symbolic, 
			]table[
			x expr=\coordindex,
			y=apg,       
			] {\datatable};%
			
			\addplot [
			cfmu,
			thick,
			mark=*,
			mark options={fill=cfmu},
			nodes near coords, 
			point meta=explicit symbolic, 
			]table[
			x expr=\coordindex,
			y=fmu,       
			] {\datatable};%
			
			\addplot [
			cfhals,
			thick,
			mark=*,
			mark options={fill=cfhals},
			nodes near coords, 
			point meta=explicit symbolic, 
			]table[
			x expr=\coordindex,
			y=fhals,       
			] {\datatable};%
			
			\addplot [
			cbcdp,
			thick,
			mark=*,
			mark options={fill=cbcdp},
			nodes near coords, 
			point meta=explicit symbolic, 
			]table[
			x expr=\coordindex,
			y=bcdp,       
			] {\datatable};%
			
			\addplot [
			ccdtf,
			thick,
			mark=*,
			mark options={fill=ccdtf},
			nodes near coords, 
			point meta=explicit symbolic, 
			]table[
			x expr=\coordindex,
			y=cdtf,       
			] {\datatable};%
			
			\addplot [
			cgcd,
			thick,
			mark=*,
			mark options={fill=cgcd},
			nodes near coords, 
			point meta=explicit symbolic, 
			]table[
			x expr=\coordindex,
			y=gcd,       
			] {\datatable};%
			
			\addplot [
			csacd,
			thick,
			mark=*,
			mark options={fill=csacd},
			nodes near coords, 
			point meta=explicit symbolic, 
			]table[
			x expr=\coordindex,
			y=sacd,       
			] {\datatable};%
			
			\addplot[
			bleu,
			thick,
			mark = triangle*,
			mark options = {fill=bleu},
			]table[
			x expr=\coordindex,
			y=fsacd,     
			]{\datatable};%
			
			\legend{APG (o.m),FMU (o.m),FHALS (o.m),BCDP,CDTF,GCD,SaCD,FSaCD}%
			
			\node (source) at (axis cs:4,4000){};%
			\node (destination) at (axis cs:4,15){};%
			\draw[black, <->](source) -- node [above, midway] {$234\times$} (destination);%
			\end{axis}
			\end{tikzpicture}			
	}}%
	\subfloat{
		\pgfplotstableread[col sep=space]{rankhpc.dat}\datatable%
		\resizebox {.45\columnwidth} {!} {
			\begin{tikzpicture}
			\begin{axis}[
			legend columns=4,transpose legend,
			xticklabels from table={rankhpc.dat}{xaxis},
			xlabel = {(d) Rank},
			xtick=data,
			ymode = log,
			ylabel={Runtime in secs},
			legend pos=north west
			]%
			
			\addplot [
			white,
			nodes near coords, 
			forget plot,
			]table[
			x expr=\coordindex,
			y=dum,     
			]{\datatable};%
			
			\addplot [
			capg,
			thick,
			mark=*,
			mark options={fill=capg},
			nodes near coords, 
			point meta=explicit symbolic, 
			]table[
			x expr=\coordindex,
			y=apg,       
			] {\datatable};%
			
			\addplot [
			cfmu,
			thick,
			mark=*,
			mark options={fill=cfmu},
			nodes near coords, 
			point meta=explicit symbolic, 
			]table[
			x expr=\coordindex,
			y=fmu,       
			] {\datatable};%
			
			\addplot [
			cfhals,
			thick,
			mark=*,
			mark options={fill=cfhals},
			nodes near coords, 
			point meta=explicit symbolic, 
			]table[
			x expr=\coordindex,
			y=fhals,       
			] {\datatable};%
			
			\addplot [
			cbcdp,
			thick,
			mark=*,
			mark options={fill=cbcdp},
			nodes near coords, 
			point meta=explicit symbolic, 
			]table[
			x expr=\coordindex,
			y=bcdp,       
			] {\datatable};%
			
			\addplot [
			ccdtf,
			thick,
			mark=*,
			mark options={fill=ccdtf},
			nodes near coords, 
			point meta=explicit symbolic, 
			]table[
			x expr=\coordindex,
			y=cdtf,       
			] {\datatable};%
			
			\addplot [
			cgcd,
			thick,
			mark=*,
			mark options={fill=cgcd},
			nodes near coords, 
			point meta=explicit symbolic, 
			]table[
			x expr=\coordindex,
			y=gcd,       
			] {\datatable};%
			
			\addplot [
			csacd,
			thick,
			mark=*,
			mark options={fill=csacd},
			nodes near coords, 
			point meta=explicit symbolic, 
			]table[
			x expr=\coordindex,
			y=sacd,       
			] {\datatable};%
			
			\addplot[
			bleu,
			thick,
			mark = triangle*,
			mark options = {fill=bleu},
			]table[
			x expr=\coordindex,
			y=fsacd,     
			]{\datatable};%
						\addplot [
			red,
			nodes near coords, 
			forget plot,
			only marks,
			point meta=explicit symbolic, 
			every node near coord/.style={anchor=-180} 
			]table[
			x expr=\coordindex,
			y=out,     
			meta index = 11,
			] {\datatable};%

			\legend{APG (o.m),FMU,FHALS,BCDP,CDTF,GCD,SaCD,FSaCD}%
			
			\node (source) at (axis cs:6,17000){};%
			\node (destination) at (axis cs:6,470){};%
			\draw[black, <->](source) -- node [below, midway] {$37.5\times$} (destination);%
			\end{axis}
			\end{tikzpicture}			
	}}%
	\caption{Scalability analysis of FSaCD and other algorithms using synthetic datasets.}%
	\label{fig:5}%
\end{figure*}
        
If the calculation of $mttkrp$ is parallelized as in the proposed algorithm FSaCD and executed using the cores of a single machine, the time taken for $mttkrp$ can be reduced as shown in  Figure~\ref{fig:5}(a). We evaluate FSaCD in terms of mode length, density, and rank of the tensor. We set the rank to $100$ and density to $0.00001$ to evaluate the performance in terms of mode length. As shown in Figure~\ref{fig:5}(b), FSaCD shows up to $8$ times fast computing comparing to GCD, the next best algorithm. By setting the rank to $100$ and mode length to $2^{13}$, we decreased the density to evaluate the performance in terms of density. As shown in Figure~\ref{fig:5}(c), FSaCD shows up to $234$ $times$ fast computing performance in comparision to GCD. It is interesting to note that SaCD and FSaCD show similar performance on sparse datasets showing that parallelization has no effect. To evaluate the performance in terms of rank, we set the density to $0.00001$ and the mode length of the tensor to $2^{13}$ and increased the rank from $10$ to $150$. FSaCD shows $37.5$ $times$ improved performance in comparision to CDTF as shown in Figure~\ref{fig:5}(d). It ascertains that FSaCD can handle higher rank easily. It is also interesting to note that FSaCD shows a converged performance for ranks where the increase in the rank does not increase the runtime. 

\section{Conclusion}
In this paper, we propose an element selection-based coordinate descent algorithm SaCD that measures the important elements for optimization using Lipschitz continuity and provides the saturated point for early stopping. The proposed Lipschitz continuity based element importance calculation introduces an additional regularity condition to the optimization process and helps to speed-up the convergence. SaCD can efficiently process NTF with higher mode length, rank, and density by reducing the frequent gradient updates.  We also extend SaCD (called FSaCD) for the parallel environment to further improve the performance.  We conducted theoretical and empirical studies to demonstrate the efficiency of SaCD. Theoretical analysis shows the time complexity and memory requirement as well as proves the fast convergence property of SaCD. Empirical analysis shows the superiority of SaCD and FSaCD in comparison to the state-of-the-art algorithms, in terms of tensor size (mode length), rank, and density without compromising the accuracy. Results show the applicability of SaCD in recommendation and pattern mining where efficiency is achieved at no cost of accuracy. SaCD and FSaCD require the element importance matrix stored in the memory which makes it challenging to extend in a distributed environment. In future work, we will explore SaCD and FSaCD to effectively handle the datasets in the distributed environment.

\section*{Acknowledgement}
We like to express our gratitude to Dr. U Kang, Associate Professor, Seoul National University for his comments that had greatly improved the manuscript. His willingness to give his time so generously is very much appreciated.

\bibliographystyle{ACM-Reference-Format}
\bibliography{SaCD}

\end{document}